\documentclass{article}

\usepackage{microtype}
\usepackage{graphicx}
\usepackage{booktabs} 

\usepackage{hyperref}

\usepackage{amsmath,amssymb,amsthm,amsfonts,bm,url,dsfont,amsthm,hyperref,color, graphicx, subcaption}
\usepackage{enumitem}

\newtheorem{theorem}{Theorem}[section]
\newtheorem{corollary}{Corollary}[theorem]
\newtheorem{lemma}[theorem]{Lemma}
\newtheorem{proposition}[theorem]{Proposition}
\newtheorem{remark}{Remark}[section]

\def\E{\mathbb{E}}
\def\P{\mathbb{P}}
\def\R{\mathbb{R}}

\newcommand{\mb}[1]{{\mathbf{{#1}}}}

\usepackage[accepted]{icml2020}

\icmltitlerunning{Estimating the Number and Effect Sizes of Non-null Hypotheses}

\begin{document}

\twocolumn[
\icmltitle{Estimating the Number and Effect Sizes of Non-null Hypotheses}



\icmlsetsymbol{equal}{*}

\begin{icmlauthorlist}
\icmlauthor{Jennifer Brennan}{uw}
\icmlauthor{Ramya Korlakai Vinayak}{uw}
\icmlauthor{Kevin Jamieson}{uw}
\end{icmlauthorlist}

\icmlaffiliation{uw}{Paul G. Allen School of Computer Science and Engineering, University of Washington, Seattle, WA}

\icmlcorrespondingauthor{Jennifer Brennan}{jrb@cs.washington.edu}

\icmlkeywords{multiple hypothesis testing}

\vskip 0.3in
]



\printAffiliationsAndNotice{}  

\begin{abstract}
We study the problem of estimating the distribution of effect sizes (the mean of the test statistic under the alternate hypothesis) in a multiple testing setting. Knowing this distribution allows us to calculate the power (type II error) of any experimental design. We show that it is possible to estimate this distribution using an inexpensive pilot experiment, which takes significantly fewer samples than would be required by an experiment that identified the discoveries. Our estimator can be used to guarantee the number of discoveries that will be made using a given experimental design in a future experiment. We prove that this simple and computationally efficient estimator enjoys a number of favorable theoretical properties, and demonstrate its effectiveness on data from a gene knockout experiment on influenza inhibition in \textit{Drosophila}.
\end{abstract}
\section{Introduction}\label{sec:intro}
Designing scientific experiments is something of a chicken and egg problem. 
In order to design an experiment with a specified power (type II error), we need to know the effect size (the mean of the test statistic under the alternate hypothesis). 
The effect size determines the required accuracy of each measurement, which increases with the number of \textit{experimental replicates} (samples).
Unfortunately, this effect size is typically unknown, and estimating the effect size for a single hypothesis test is as sample intensive as performing the original experiment. 
In the case of single hypothesis testing, this presents a fundamental barrier to efficient experimental design. 
By contrast, in the setting of multiple hypothesis testing, we show that it is possible to estimate the distribution of effect sizes present in the data using an inexpensive pilot experiment, which takes significantly fewer samples than would be required for the full experiment.

\begin{figure*}
    \centering 
    \includegraphics[trim={0.5cm 5cm 2.2cm 2cm},clip,width=\textwidth]{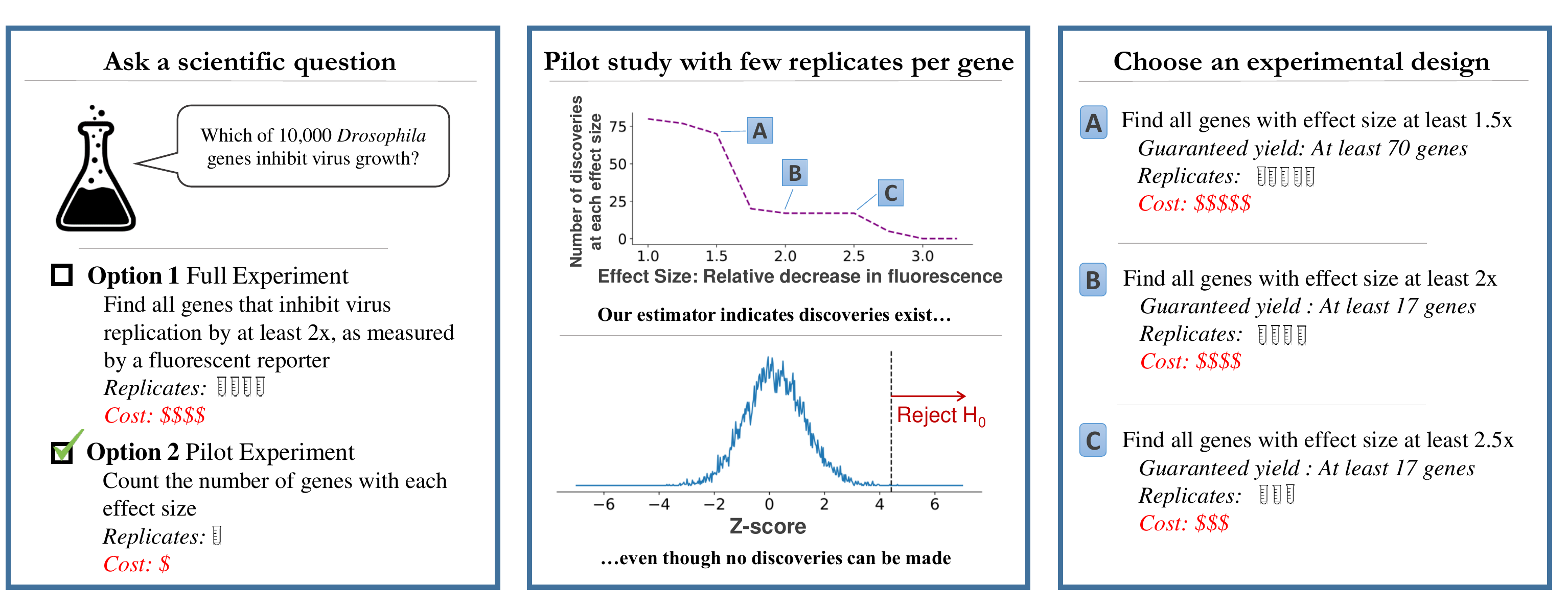}
    \caption{{
    When applied to the results of a pilot experiment, our estimator can estimate the cost and number of discoveries guaranteed by different experimental designs. In this example, the original experiment design (Option 1) is expensive, with no guarantee on the number of discoveries that will be made. Our method suggests two alternatives to the original experimental design (B); the same guarantee on discoveries could be made at lower cost (C), or additional discoveries could be made at higher cost (A).}
    }
    \label{fig:vizAbstract}
\end{figure*}
For example, suppose a scientist would like to test 10,000 genes using an experimental measurement that is distributed $\mathcal{N}(\mu_i, \tfrac{1}{t})$ when the effect size is $\mu_i$ and $t$ replicates are performed. Without knowledge of the likely effect sizes, it is unclear how to choose an experimental design. An experiment with too many replicates per hypothesis is wasteful; one with too few will lack the statistical power to identify alternate hypotheses. 
In this paper, we seek to facilitate experimental design in the multiple testing setting by answering the question \emph{``How many hypotheses have an effect size of at least $\gamma$?''} using significantly fewer samples than would be needed to identify all discoveries with that effect size. 
These estimates suggest a trade-off between the cost of an experiment (as measured by the number of experimental replicates required to achieve a certain power) and the number and effect sizes of the discoveries that will be made. 
Figure \ref{fig:vizAbstract} illustrates the application of our estimator to an inexpensive pilot study, allowing a scientist to evaluate possible experimental designs.
The application to experimental design motivates an important property of our estimator: it must produce a \textit{conservative} estimate of the number of hypotheses above a given effect size. If the scientist designs a costly experiment based on the results of this estimator, it is important to ensure that this experiment will generate at a minimum the estimated number of discoveries.

As a baseline, one approach to this estimation problem is to use a \textit{plug-in estimator}, which estimates the entire distribution of effect sizes and then ``plugs in'' this estimate as if it were the true distribution. The plug-in estimator could start with the maximum likelihood estimate (MLE) of the distribution of effect sizes given the observed test statistics. The estimate for the fraction of hypotheses above some effect size $\gamma$ would simply be the fraction of this distribution that exceeded $\gamma$. 
Unfortunately, such a plug-in estimator based on the MLE may vastly overestimate this fraction, as two distributions can have similar likelihoods but very different amounts of mass above some threshold.

In this work, we design an estimator for the fraction of hypotheses with effect sizes above a given threshold, for all thresholds simultaneously. Our estimator operates in the spirit of the Kolmogorov-Smirnov test, first creating an $\ell_\infty$ ball around the empirical CDF to define plausible distributions, and then finding the element of the ball with the smallest amount of probability mass above $\gamma$. 
With high probability, this amount of mass does not exceed the true fraction of hypotheses with mean at least $\gamma$.
We prove that this simple and computationally efficient estimator enjoys a number of favorable theoretical properties, including finite-sample upper and lower bounds on the value of the estimate. 

\subsection{Problem Statement}
Let $\nu_*$ be a distribution on $\mathbb{R}$,
and for $i=1,2,\dots n$ let 
\begin{align*}
    \mu_i \sim \nu_*
\end{align*}
be an unobserved latent variable drawn iid from $\nu_*$. For each $\mu_i$ drawn from $\nu_*$, we observe the test statistic
\begin{align*}
    X_i \sim f_{\mu_i},
\end{align*}
where $f_\mu$ is a known distribution parameterized by the effect size $\mu$. 
For example, suppose the test statistics were Z-scores, which are distributed according to the standard normal distribution under the null hypothesis and shifted by the effect size under the alternate. Then, $f_\mu = \mathcal{N}(\mu, 1)$.
While our estimator is well defined for any parametric $f$ (e.g., any single-parameter exponential family), we focus on Gaussian test statistics for exposition. In the setting of Figure \ref{fig:vizAbstract}, $\nu_*$ represents the distribution of effect sizes and $X_i$ are the observations.

Our goal is to estimate the probability that the effect size of an observation is greater than $\gamma$,
\begin{align}
    \zeta_{\nu_*}(\gamma) := \mathbb{P}_{\nu_*}(\mu > \gamma), \label{eqn:dfnZeta}
\end{align}
simultaneously for all $\gamma\in\mathbb{R}$. 

The problem of counting the non-null hypotheses is most interesting when $\gamma$ is small. For example, consider the case when the test statistics $X_i$ are Z-scores. 
Under the hypothesis that all effect sizes are zero, the expected maximum Z-score is $\mathbb{E}[\max_i X_i] \approx \sqrt{\log n}$. Therefore, if we want to avoid any false discoveries, we cannot reject any hypotheses with test statistic less than $\Theta(\sqrt{\log n})$. If the effect sizes are at least this large, then we will be able to identify the alternate hypotheses through a standard Bonferroni correction \cite{dunn1961multiple}. In this regime, counting is no more difficult than identification. However, if the effect sizes are much smaller than this threshold (say, if all $\mu_i \ll 1$), identification could be impossible. Our estimator, by contrast, detects the existence of discoveries even in this low signal-to-noise regime.
\subsection{Contributions}
Our contributions are as follows:
\begin{itemize}[leftmargin=*, topsep=-2pt, itemsep=0pt]
\item Given a parameterization $f_\mu$, we propose an estimator that provides a conservative estimate of the fraction of effect sizes above a given threshold, simultaneously for all thresholds (Section \ref{sec:estimator}).
\item We provide finite-sample bounds on the error of our estimator (Theorem \ref{lem:est-finiteSample}).
\item In the low signal-to-noise regime and the setting of Gaussian mixtures, we compare our estimator's sample complexity to a known lower bound for hypothesis testing (detecting the presence of the alternate hypothesis), and we give a novel lower bound for the sample complexity of estimation (estimating the fraction of means from the alternate hypothesis). We show that our method matches finite-sample rates for these problems, even though it is designed for more general distributions than the ones in these lower bounds (Section \ref{sec:estimator}).

\item We describe how to use this estimator to design pilot studies for scientific experimentation (Section \ref{sec:preScreen}).
When testing $n$ hypotheses in the low signal-to-noise regime,
our technique detects treatments with positive effect sizes using a factor of $n$ fewer replicates than it would take to identify them. Additionally, the results of the pilot experiment can be used to upper bound the cost of identifying the discoveries at each effect size.
\end{itemize}

\subsection{Related Work}\label{sec:relatedworks}
The problem of estimating the number of null hypotheses has been studied extensively in the statistics literature. Our goal in this work is to provide a conservative estimate of the number of hypotheses with effect size above some threshold (Eqn \eqref{eqn:dfnZeta}). There are several lines of work related to this goal.

\textbf{Simple Null Hypotheses}
A different but related problem is to estimate the number of non-null hypotheses, regardless of their effect sizes, i.e., $\P_{\nu_*}(\mu \neq 0)$. In this setting - also known as the simple null hypothesis - it is possible to compute $p$-values that are uniformly distributed under the null.
For example, when observations are drawn $X_i \sim \mathcal{N}(\mu_i, 1)$, the $p$-value is $p_i=1-\Phi(X_i)$, where $\Phi$ is the standard normal CDF.

The graphical estimator of Schweder \& Spj{\o}tvoll \yrcite{schweder1982plots} was the first technique to estimate the number of nulls, using the principle that $p$-values are distributed uniformly under the null hypothesis and skewed toward zero under the alternate. Their technique estimates the density of the $p$-value distribution at 1. 
This same idea was improved in the context of estimating the number of nulls for adaptive control of the false discovery rate (FDR) \cite{benjamini2000adaptive, storey2002direct}.
These later works provide finite-sample guarantees on overestimating the number of nulls in order to make non-asymptotic guarantees on FDR control. However, none of these results provide lower bounds on the estimated number of non-nulls. Motivated by adaptive FDR control, techniques for counting the number of non-null hypotheses have been extended to incorporate prior knowledge about the dependence structure of the hypotheses or the likelihood that each test will result in a discovery. See Li \& Barber \yrcite{li2019multiple} for a review of this area.

Bounds on the False Discovery Proportion - the high-probability analogue of FDR - can also be employed to report a guarantee on the number of significant effects. The simultaneous FDP estimator of Katsevich and Ramdas \yrcite{katsevich2018towards} provides such bounds simultaneously for all sets in a path. A guarantee on the FDP of a set corresponds to a lower bound on the number of discoveries; maximizing over the guarantees provided by each set in the path gives an improved lower bound. With an assumption on the form of the test statistic under the alternate hypothesis, this algorithm can be modified to bound the number of discoveries above an arbitrary threshold. We compare to this baseline method in our experimental results.

Another technique for the simple null setting, again motivated by the uniform distribution of $p$-values under the null, is to test the extent to which the distribution of $p$-values deviates from the uniform distribution. 
Several estimators have taken this approach \cite{genovese2004stochastic, meinshausen2006estimating, patra2016estimation, jin2008proportion}. Most similar to our work are the techniques that build one-sided confidence intervals around the empirical CDF of $p$-values \cite{genovese2004stochastic, meinshausen2006estimating}, which provide finite-sample error bounds and a conservative estimator. Finally, there are estimators specific to the Gaussian setting, which estimate the zero-mean component in a mixture of Gaussians \cite{cai2007estimation, carpentier2019adaptive}.

Extensions to one-sided null hypotheses ($H_0:~\mu \leq 0$) further assume that $p$-values are subuniformly distributed when $\mu < 0$ \cite{meinshausen2005lower, li2019multiple} or assume a gap between $0$ and the smallest alternate effect size \cite{lee2019uncertainty}.
These works estimate the quantity $\P_{\nu_*}(\mu > 0).$
This problem is a special case of ours, because subuniformity holds only for the threshold of $\gamma=0$. 

\textbf{Composite Null Hypotheses}
We seek to estimate the number of hypotheses with an effect size above some threshold.  Here, $p$-values are neither subuniform nor necessarily well defined, so much of the previous work is not applicable.
The Fourier transform technique \cite{jin2008proportion} can be extended to address composite null hypotheses \cite{chen2019a}. However, this extension only provides asymptotic results, which are insufficient since we seek a conservative estimator.

\textbf{Adapting the Generalized Likelihood Ratio Test} Jiang \& Zhang \yrcite{glrt} develop asymptotic power statements for the generalized likelihood ratio test for Gaussian observations. We discuss in Section \ref{sec:future} how this work could be used to create an estimator for our problem, and highlight the limitations that make this approach impractical.

\textbf{Plug-in Estimation} As discussed in Section \ref{sec:intro}, another approach to this problem is plug-in estimation, where an estimate $\widetilde{\nu}$ of the distribution $\nu_*$ is used to form an estimator $\widehat{\zeta}_n(\gamma) = \P_{\widetilde{\nu}}(\mu > \gamma)$. When $f_\mu$ is Gaussian, the task is to learn a mixture of Gaussians. In this setting, much effort has been devoted to recovering the mixture parameters \cite{pearson1894contributions, belkin2010polynomial, kalai2010efficiently, hardt2015tight} or learning a mixture that is close to the original distribution in some metric, such as total variation (TV) distance \cite{moitra2010settling, daskalakis2014faster}. Outside of the Gaussian setting, recent works have provided guarantees for learning mixtures of binomial distributions in terms of the Wasserstein-1 distance \cite{tian2017learning, vinayak2019maximum}. These types of theoretical guarantees do not lend themselves easily to guarantees on our problem, since two distributions can be close in TV or Wasserstein distance but have very different amounts of mass above some threshold $\gamma$.

\textbf{Empirical Bayes Methods} Our estimator takes advantage of multiple hypothesis testing by using the empirical distribution of the $X_i$ to learn something about the latent distribution $\nu_*$. The same idea can be seen in empirical Bayes methods, where the empirical distribution of $X_i$ is used as the prior over $X$. Several papers have taken an empirical Bayes approach to multiple testing, but none address our exact problem. Efron \yrcite{efron2007size} uses an empirical Bayes method to estimate the distribution of $X$ under the alternate hypothesis, which is distinct from our goal of estimating $\nu_*$ (note we cannot simply deconvolve Efron's estimate to get $\nu_*$, as it is not guaranteed to have any parametric form).
Stephens \yrcite{stephens2017false} uses empirical Bayes methods and a strong unimodality assumption on $\nu_*$ to produce estimates and confidence intervals for each $\mu_i$. While these confidence intervals could theoretically be used to estimate \eqref{eqn:dfnZeta}, the fact that Stephens' method produces a confidence interval for individual $\mu_i$ suggests that they will be too loose to compete with our method.
Indeed, we see this looseness in the experimental results, where our estimator outperforms Stephens' in our regime of interest.
Furthermore, this method only works for Gaussian and t-distributed observations.

\section{Estimating Effect Sizes}\label{sec:estimator}
Recall our goal, to estimate $\zeta_{\nu_*}(\gamma)$ from Eqn \eqref{eqn:dfnZeta}.
Let  $\widehat{F}_n(t) = \frac{1}{n} \sum_{i=1}^n \mb{1}\{ X_i \leq t \}$ be the empirical CDF of the test statistics $X_i$ and
\begin{align*}
    F_\nu(t) &= \P_{\mu \sim \nu,~X\sim f_\mu}( X \leq t)
\end{align*}
be the true CDF of test statistics under latent distribution $\nu$. For any $\gamma \in \R$, our estimator is given by
\begin{align}
    \widehat{\zeta}_n(\gamma) &= \min_{\nu : ||\widehat{F}_n - F_\nu ||_\infty \leq \tau_{\alpha, n}} \int_\gamma^\infty \nu(x)dx \label{eqn:estimator}
\end{align}
where the estimator is conservative with probability at least $1-\alpha$, and
\begin{align*}
    \tau_{\alpha, n} = \sqrt{\frac{\log(2/\alpha)}{2n}}.
\end{align*}
The intuition for this estimator is as follows.
To conservatively estimate the amount of mass $\zeta$ above threshold $\gamma$, we look for the distribution with the smallest amount of mass above $\gamma$ that could have plausibly generated the observations $X_i$. Our measurement of plausibility is based on high probability bounds on the deviation between the empirical CDF and its expectation. If $\widehat{F}_n$ was in fact drawn from $F_\nu$, then with high probability the $\ell_\infty$ distance between $\widehat{F}_n$ and $F_\nu$ will not exceed $\tau_{\alpha, n}$. By restricting our search space to the $\ell_\infty$ ball around $\widehat{F}_n$ (seen in the constrained optimization from Eqn \eqref{eqn:estimator}), we do not overestimate the true amount of mass above $\gamma$, with high probability.
Moreover, using different values of $\gamma$ traces a curve for $\zeta_{\nu_*}(\gamma)$ (see the middle panel of Figure \ref{fig:vizAbstract}). We note that this estimator can be implemented as an efficient convex program. We simply discretize $x$ over some range, and the estimator becomes a convex program in the vector $\mathbf{x}$. It can then be solved using off-the-shelf software (see Appendix \ref{app:implementation} for details). 

\subsection{Main Results}
Our estimator underestimates the true mass $\zeta_{\nu_*}(\gamma)$ for all $\gamma$ simultaneously with high probability. 
Furthermore, we provide a finite sample bound on how much we underestimate $\zeta_{\nu_*}(\gamma)$ at every $\gamma$. 
\begin{theorem}\label{lem:est-finiteSample}
For $i=1,\dots,n$, let $\mu_i\sim\nu_*$ and $X_i \sim f_{\mu_i}$ where each draw is iid. Let our simultaneous estimator be given by (\ref{eqn:estimator}). Then, the probability of overestimating the fraction of hypotheses with effect size above any threshold $\gamma$ is bounded by $\alpha$:
\begin{align*}
    \P\left(\exists \gamma : \widehat{\zeta}_n(\gamma) > \zeta_{\nu_*}(\gamma)\right) \leq \alpha.
\end{align*}

Furthermore, with probability at least $1-\delta$, for all $\gamma\in\R$ and $\varepsilon\in(0, \zeta_{\nu_*}(\gamma)]$ we have $\zeta_{\nu_*}(\gamma) - \hat{\zeta}_n(\gamma) \leq \varepsilon$
whenever
\begin{align}
    n \geq \frac{\log\left(\frac{4}{\alpha\delta}\right)}{\left(\min_{\nu : \P_\nu((\gamma, \infty)) \leq \zeta_{\nu_*}(\gamma) - \varepsilon}||F_\nu - F_{\nu_*}||_\infty\right)^2}.\label{eqn:generalSampleComplexity}
\end{align} 
\end{theorem}

\begin{remark}[Pointwise consistency] Our estimator is pointwise consistent. For any threshold $\gamma$ and any $\varepsilon > 0$, there is some $n$ large enough that the error in our estimate satisfies $\zeta_{\nu_*}(\gamma) - \hat{\zeta}(\gamma) < \varepsilon$. This follows from the fact that, for any $\varepsilon > 0$, the denominator of \ref{eqn:generalSampleComplexity} is strictly positive.
\end{remark}

Our estimator is guaranteed not to overestimate $\zeta_{\nu_*}(\gamma)$, which
is critical in the use of pilot studies to guide experimental design. The key quantity in this sample complexity result is the minimum $\ell_\infty$ distance between the true CDF $F_{\nu_*}$ and the set of CDFs corresponding to mixing distributions with less than $\zeta$ mass above $\gamma$. We call this set of mixing distributions $S$,
\begin{align}
S(\zeta, \gamma) :=\{ \nu : \P_\nu((\gamma, \infty)) \leq \zeta\}.\label{eqn:defineS} 
\end{align}
Specifically, consider $S(\zeta_{\nu_*}(\gamma) - \varepsilon, \gamma)$, which appears in Eqn \eqref{eqn:generalSampleComplexity}. If $\varepsilon = 0$, then we have $\nu_* \in S(\zeta_{\nu_*}, \gamma)$, so the minimum $\ell_\infty$ distance to $F_{\nu_*}$, $\min_{\nu \in S(\zeta_{\nu_*}(\gamma), \gamma)}||F_\nu - F_{\nu_*}||_\infty$, would be zero, implying that no finite sample can guarantee $\varepsilon = 0$. This reflects the fact that $\widehat{\zeta}_n(\gamma)$ is an underestimate at every $\gamma$; therefore, in order for $\varepsilon$ to be zero, we must have estimated $\zeta_{\nu_*}(\gamma)$ exactly. As $\varepsilon$ increases, $S(\zeta_{\nu_*}(\gamma) - \varepsilon, \gamma)$ shrinks, and the distance to $F_{\nu_*}$ increases, decreasing the required number of samples $n$. 

To interpret the sample complexity in Theorem \ref{lem:est-finiteSample}, we consider a simple model where test statistics are drawn from a mixture of two Gaussians. In this setting, which we denote $X_i\sim P(\zeta_*, \gamma_*)$, we have
\begin{equation}\label{eq:canonicalTwoSpikes}
\begin{split}
    \mu_i &\sim (1-\zeta_*)\delta_0 + \zeta_*\delta_{\gamma_*}\\
    X_i &\sim \mathcal{N}(\mu_i, \sigma^2),
\end{split}
\end{equation}
where $\delta_x$ is the Dirac delta function at $x$. There are two natural questions we might ask: How many samples a re necessary to determine the existence of the mixture component at $\gamma_*>0$, and how many samples are required to estimate the weight of this component? We call these the \textit{testing} and \textit{estimation} problems respectively. In the following sections, we address our algorithm's sample complexity for these problems, and compare to lower bounds. For ease of exposition, let $\alpha=\delta$, although the results hold for the more general case.

\subsection{Global Null Testing}\label{sec:globalNull}
In the global null testing problem, we observe $X_i$ according to \eqref{eq:canonicalTwoSpikes}, and we want to determine whether $\zeta_* > 0$ (i.e., testing $H_0:~P_{\nu_*}(\mu>0)=0$ vs $H_1:~P_{\nu_*}(\mu > 0) > 0$). 
Our test declares $H_1$ if $\widehat{\zeta}_n(0)>0$, and $H_0$ if $\widehat{\zeta}_n(0) = 0$.
Clearly this test erroneously declares $H_1$ with probability at most $\delta$ (it has type I error at most $\delta$), since $\widehat{\zeta}_n(0)\leq \zeta_*$ with probability at least $1-\delta$ 
(recall that we set $\alpha=\delta$ in Theorem \ref{lem:est-finiteSample}). 
The next corollary bounds the sample complexity that guarantees a probability of detection of at least $1-\delta$ (i.e., that bounds the type II error by $\delta$).
\begin{corollary}\label{cor:est-finiteSampleGauss}
Let $\{X_i\}_{i=1}^n$ be drawn according to \eqref{eq:canonicalTwoSpikes}. Consider the simultaneous estimator $\widehat{\zeta}_n$ defined by (\ref{eqn:estimator}). 
Then, with probability at least $1 - \delta$, we have $\widehat{\zeta}_n(0) \leq \zeta_*$ and $\widehat{\zeta}_n(0) > 0$ whenever
\begin{align*}
    n \geq \frac{2\log\left(\frac{2}{\delta}\right)}{\zeta_*^2\left( \Phi_\sigma\left(\frac{1}{2}\gamma_*\right) - \Phi_\sigma\left(-\frac{1}{2}\gamma_*\right) \right)^2},
\end{align*}
where $\Phi_\sigma$ is the CDF of the distribution $\mathcal{N}(0, \sigma^2)$. Furthermore, if $\gamma_* < \sigma$, then the above can be simplified to
\begin{align*}
    n \geq \frac{16\sigma^2\log\left(\frac{2}{\delta}\right)}{\zeta_*^2\gamma_*^2}.
\end{align*}
\end{corollary}

\textit{Proof Sketch.}
To obtain this sample complexity result, we must lower bound the distance term in the denominator of Eqn \eqref{eqn:generalSampleComplexity}. Recalling our definition of $S$ in Eqn \eqref{eqn:defineS}, we lower bound the associated minimax quantity by its value at a specific point, $t=\tfrac{1}{2}\gamma_*$,
\begin{align*}
    \min_{\nu\in S(0,0)}||F_\nu - F_{\nu_*}||_\infty
    &= \min_{\nu\in S(0,0)} \sup_{t\in\R}|F_\nu(t) - F_{\nu_*}(t)|\\
    &\geq \min_{\nu\in S(0,0)} F_\nu(\tfrac{1}{2}\gamma_*) - F_{\nu_*}(\tfrac{1}{2}\gamma_*)
\end{align*}
The constraint $\nu\in S(0,0)$ allows us to lower bound the first quantity by $\Phi(\tfrac{1}{2}\gamma_*)$, and we compute the second quantity exactly, giving the first conclusion of the corollary. The second conclusion follows from a quadratic approximation to the normal density. \qed

We compare Corollary \ref{cor:est-finiteSampleGauss} to the finite-sample lower bound arising from the ``most biased coin problem'' \cite{chandrasekaran2014finding, MostBiasedCoin}. In this problem, the algorithm draws $N$ observations $X_i$ as per \eqref{eq:canonicalTwoSpikes}, where $N$ is potentially a random variable, according to either $H_0: X_i\sim \mathcal{N}(0,\sigma^2)$ or $H_1: X_i\sim P(\zeta_*, \gamma_*)$.
When $\gamma_*$ and $\zeta_*$ are known and $\gamma_* \leq \sigma$, Theorem 2 of Jamieson et al. \yrcite{MostBiasedCoin} states that any (potentially randomized) procedure that decides between these hypotheses with probability of error at most $\delta$ requires at least
\begin{align*}
    \E[N] \geq \max\left\{ \frac{1-\delta}{\zeta_*}, ~ \frac{\sigma^2\log(1/\delta)}{2\zeta_*^2\gamma_*^2} \right\}
\end{align*}
samples. To facilitate comparison with the sample complexity of our estimator, we show in Lemma \ref{lem:MBC-algebra} that the small-$\gamma_*$ sample complexity from Corollary \ref{cor:est-finiteSampleGauss} matches the stated lower bound up to constants both when $\delta$ is fixed and as $\delta\to 0$. 

\begin{figure*}
    \centering 
    \includegraphics[trim={0 7.5cm 0.8cm 3.4cm}, clip,width=0.9\textwidth]{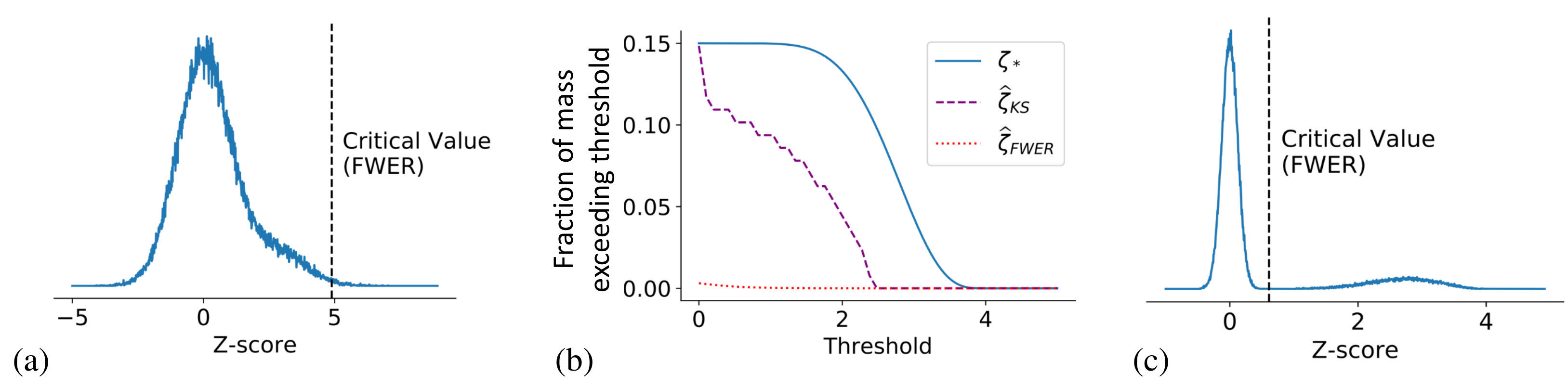}
    \caption{
    Our estimator applied to pilot experiments. (a) After observing $X_i\sim\mathcal{N}(\mu_i, 1)$ for $i=1,\ldots,n$ with $n=10^4$, only $0.3$\% of null hypotheses are rejected via a Bonferroni corrected test (indicated by the FWER critical value). However, the Z-scores appear skewed positive, suggesting additional discoveries exist. (b) Our estimator $\widehat{\zeta}_{KS}$ indicates that there are many discoveries to be made; for example, at least $9$\% of treatments have effect size at least $1$, and at least $4\%$ have effect size at least $2$. Note that our estimator also counts more discoveries at each threshold than are identified by Bonferroni correction ($\widehat{\zeta}_{FWER}$), without exceeding the true value $\zeta_*$. (c) The experimenter designs an experiment to identify the effects greater than $2$, and allocates $\gamma^{-2}\log(n)\log(1/\widehat{\zeta}(\gamma))=8$ replicates per hypothesis. Now, $14$\% of the null hypotheses can be rejected.}
    \label{fig:prescreen_new}
\end{figure*}
\subsection{The Estimation Problem}
In the estimation problem, we observe $X_i$ according to \eqref{eq:canonicalTwoSpikes}, and we estimate $\zeta_*$ using our estimator $\widehat{\zeta}(0)$. Since $\widehat{\zeta}(0) \leq \zeta_*$ with high probability, it remains to understand the magnitude of this underapproximation $-$ the dependence of $\varepsilon$ from Theorem \ref{lem:est-finiteSample} on the number of samples $n$. The following corollary describes the number of samples needed to guarantee an error bound $\varepsilon \leq \tfrac{1}{2}\zeta_*$ with high probability.
\begin{corollary}\label{cor:estimation-finiteSampleGauss}
    Let $\{X_i\}_{i=1}^n$ be drawn according to \eqref{eq:canonicalTwoSpikes}.
    Let $\zeta_* > 0$ and $\gamma_*\in (0, \sigma]$. Then, with probability at least $1-\delta$, our estimate $\widehat{\zeta}_n$ from \eqref{eqn:estimator} satisfies $\widehat{\zeta}_n(0) \in (\tfrac{1}{2}\zeta_*, \zeta_*]$ as long as
    \begin{align*}
        n \gtrsim \frac{ \sigma^4 \log\left( \tfrac{2}{\delta}\right)}{\zeta_*^2\gamma_*^4}.
    \end{align*}
\end{corollary}

\textit{Proof Sketch.} 
Again, we obtain this result by lower bounding the denominator of Eqn \eqref{eqn:generalSampleComplexity}. We note that this quantity is the optimal value of a convex optimization problem, and show that the associated optimal point is
\begin{align*}
    \nu_{OPT} &= (1-\tfrac{1}{2}\zeta_*)\delta_0 + \tfrac{1}{2}\zeta_*\delta_{2\gamma_*}.
\end{align*}
Finally, we apply several Taylor series approximations to bound the optimal value as a polynomial in $\zeta_*$ and $\gamma_*$.
\qed

We present a novel lower bound for the estimation problem which matches our result up to constants.
\begin{lemma}\label{lem:finiteSampleEstBound}
Consider data $\{X_i\}_{i=1}^n$ generated under the model $X_i\sim P(\zeta, \gamma)$, parameterized by $\zeta\in (0, \tfrac{1}{2})$ and $\gamma\in (0,\sigma)$ according to our canonical two-spike model \eqref{eq:canonicalTwoSpikes}. Fix a parameterization $(\zeta_*, \gamma_*)$. For any $\varepsilon \in (0, \tfrac{2}{3}\zeta_*)$, define the set $A_\varepsilon$ of nearby parameterizations as 
\begin{align*}
    A_\varepsilon &= \left\{ (\zeta, \gamma) ~:~ |\zeta_* - \zeta| \leq 4\varepsilon,~ \tfrac{1}{3}\gamma_* \leq \gamma \leq 3\gamma_* \right\}.
\end{align*}
Suppose $\widehat{\zeta}_n(X)$ is an estimator of $\zeta$ satisfying $\P( | \widehat{\zeta}_n(X) - \zeta | \geq \varepsilon ) < \frac{1}{4}$
for any $(\zeta, \gamma)\in A_\varepsilon$. Then the estimator requires at least
$
    n \gtrsim \frac{\sigma^4}{\varepsilon^2\gamma_*^4}
$
samples on the instance $(\zeta_*,\gamma_*)$.
\end{lemma}
Instantiating Lemma \ref{lem:finiteSampleEstBound} with $\varepsilon = \tfrac{1}{2}\zeta_*$, we see that the sample complexity in Corollary \ref{cor:estimation-finiteSampleGauss} matches the lower bound.

Finally, we remark that estimating the fraction of observations with mean $\gamma_*$ to constant multiplicative error requires a factor of $\frac{\sigma^2}{\gamma_*^2}$ more samples than testing whether any observations have mean $\gamma_*$.

\subsection{Proof of Theorem~\ref{lem:est-finiteSample}}
We begin with the first part of the theorem, which bounds the probability of overestimating $\zeta_*$.
Let $A$ be the event that $\widehat{F}_n$ stays within its Dvoretzky-Kiefer-Wolfowitz (DKW) confidence interval, i.e.,
\begin{align*}
    A &:= \big\{ ||F_{\nu_*} - \widehat{F}_n ||_\infty \leq \tau_{\alpha, n} \big\}.
\end{align*}
By the DKW inequality \cite{massart1990tight}, we have $P(A^c) \leq \alpha$.
If we assume that event $A$ holds, then
\begin{align}
    \widehat{\zeta}_n(\gamma) &=  \max\big\{ \zeta \geq 0 : \min_{\nu \in S(\zeta, \gamma)} ||F_\nu - \widehat{F}_n ||_\infty > \tau_{\alpha, n}  \big\} \nonumber\\
    &\overset{(a)}{\leq} \min \big\{ \zeta \geq 0 : \min_{\nu \in S(\zeta, \gamma)} ||F_\nu - \widehat{F}_n ||_\infty \leq \tau_{\alpha, n}  \big\}\nonumber\\
    &\overset{(b)}{\leq} \zeta^*(\gamma) \nonumber 
\end{align}
where (a) holds because
$S(\zeta,\gamma)\subseteq S(\zeta', \gamma)$ for all $\zeta \leq \zeta'$, and (b) 
is true because, by event A, $\zeta_*(\gamma)$ is a member of the set. We note that on event $A$, this argument holds for all $\gamma$. We conclude that, with probability at least $1-\alpha$, we have $\widehat{\zeta}_n(\gamma) \leq \zeta^*(\gamma)$ for all $\gamma$.

To prove the power statement, let $B$ be as follows
\begin{align*}
    B &:= \big\{ ||F_{\nu_*} - \widehat{F}_n ||_\infty \leq \tau_{\delta, n} \big\}.
\end{align*}
Suppose that $B$ holds. Then,
\begin{align*}
    &\widehat{\zeta}_n(\gamma)
    = \max\big\{ \zeta\geq 0 : \min_{\nu\in S(\zeta, \gamma)}||F_\nu - \widehat{F}_n||_\infty \geq \tau_{\alpha, n} \big\}\\
    &\overset{(a)}{\geq} \max \big\{ \zeta\geq 0 : \min_{\nu\in S(\zeta, \gamma)}||F_\nu - F_{\nu_*}||_\infty  \\
    &\qquad \qquad \qquad\qquad\qquad\quad - ||\widehat{F}_n - F_{\nu_*}||_\infty \geq \tau_{\alpha,n}\big\}\\
    &\overset{(b)}{\geq} \max\big\{ \zeta\geq 0 : \min_{\nu\in S(\zeta, \gamma)}||F_\nu - F_{\nu_*}||_\infty \geq \tau_{\alpha, n} + \tau_{\delta, n} \big\}\\
    &\geq \max \big\{  \zeta\geq 0 : \min_{\nu\in S(\zeta, \gamma)}||F_\nu - F_{\nu_*}||_\infty \geq \sqrt{\tfrac{\log(4/\alpha\delta)}{n}} \big\},
\end{align*}
where (a) is the triangle inequality and (b) applies event $B$. We can rewrite the assumption on $n$ in the theorem statement as
\begin{align*}
    \min_{\nu\in S(\zeta_*-\varepsilon, \gamma)}||F_\nu - F_{\nu_*}||_\infty \geq \sqrt{\tfrac{\log(4/\alpha\delta)}{n}}
\end{align*}
which gives us $\widehat{\zeta}_n(\gamma) \geq \zeta_* - \varepsilon$, 
simultaneously for all $\gamma$, whenever $B$ holds. Since $B$ holds with probability at least $1-\delta$, this completes the proof. \qed

\section{Applications to Pilot Experiments}\label{sec:preScreen}
In this section, we consider the task of analyzing a pilot experiment to count the number of discoveries when the effect sizes are small, say $\mu < 1$ for the case of Gaussian $\mathcal{N}(\mu, 1)$ observations. 
Figure \ref{fig:prescreen_new} illustrates this process. First, the scientist allocates a small number of replicates to a large number of hypotheses in order to obtain many noisy estimates of effect sizes (Fig. \ref{fig:prescreen_new}(a)). The scientist then uses our estimator to obtain a guarantee on the number of discoveries to be made at each effect size (Fig. \ref{fig:prescreen_new}(b)). Finally, the scientist calculates the cost of identifying the discoveries that have been detected using a choice of fixed and sequential experimental designs. When the full experiment is run, it results in at least as many discoveries as our estimator has guaranteed (Fig. \ref{fig:prescreen_new}(c)).

The following proposition describes our estimator's performance on pilot data in the low signal-to-noise regime. In particular, if the pilot study design allocates its replicates equally across all $n$ hypotheses, our estimator detects the alternate hypotheses using a factor of $n$ fewer total replicates than it would take to identify these discoveries.

\begin{proposition}\label{prop:prescreen}
Consider a pilot experiment for $n$ hypotheses, where an initial budget of $B=mt$ will be used to uniformly allocate $t$ replicates to each of $m\leq n$ randomly chosen hypotheses. 
Suppose the true distribution of effect sizes is $\nu_* = (1-\zeta_*)\delta_0 + \zeta_*\delta_{\gamma_*}$ and $f_\mu = \mathcal{N}(\mu, \tfrac{1}{t})$, as when computing Z-scores from $t$ replicates. 
Then,
    \begin{align*}
        \P(\widehat{\zeta}_n(0) > 0) \geq 1-\delta
    \end{align*}
    i.e., we detect the presence of positive effects with high probability, as long as
    \begin{align*}
        \gamma_* \geq 4\sqrt{\frac{\log(\frac{2}{\delta})}{\zeta_*^2 B}}
    \qquad \text{and} \qquad
        m \geq \frac{4\log(\frac{2}{\delta})}{\zeta_*^2}.
    \end{align*}
\end{proposition}

\begin{remark}\label{remark:nTimesBetter}
Consider the setting where the budget is constrained, say $B \lesssim n$, and let $\zeta_*$ be constant (so that the proportion of discoveries does not vanish with $n$). Proposition \ref{prop:prescreen} suggests maximizing $m$: either taking $m=n$ if $B \geq n$ or taking $t=1$ if $B < n$. With this budget allocation, our estimator detects the existence of alternate effects with just $B\approx\log(\tfrac{1}{\delta})\gamma_*^{-2}$ total replicates. Note that distinguishing observations from $\mathcal{N}(0, 1)$ and $\mathcal{N}(\gamma_*, 1)$ with probability $1-\delta$ requires $\log(\tfrac{1}{\delta})\gamma_*^{-2}$ samples, so identifying all of the discoveries requires at least $n\log(\tfrac{1}{\delta})\gamma_*^{-2}$ replicates total. We conclude that, in this instance, any identification procedure requires $n$ times more total replicates than our estimator requires for detection.
\end{remark}

Our estimator can also be used to choose between a fixed experimental design (in which each hypothesis is tested with a fixed number of replicates) and a sequential design (in which the next replicate is allocated after observing all previously drawn $X_i$). A sequential design, as in Jamieson \& Jain \yrcite{BanditFDR}, can be more difficult to implement, but could result in significant savings if the 
alternate effect sizes span a large range.
By providing information about the variety of effect sizes, our estimator quantifies the advantage of using a sequential experimental design.

\begin{figure}
    \centering
    \includegraphics[trim={0cm 0 0 0.5cm},clip, width=0.9\columnwidth]{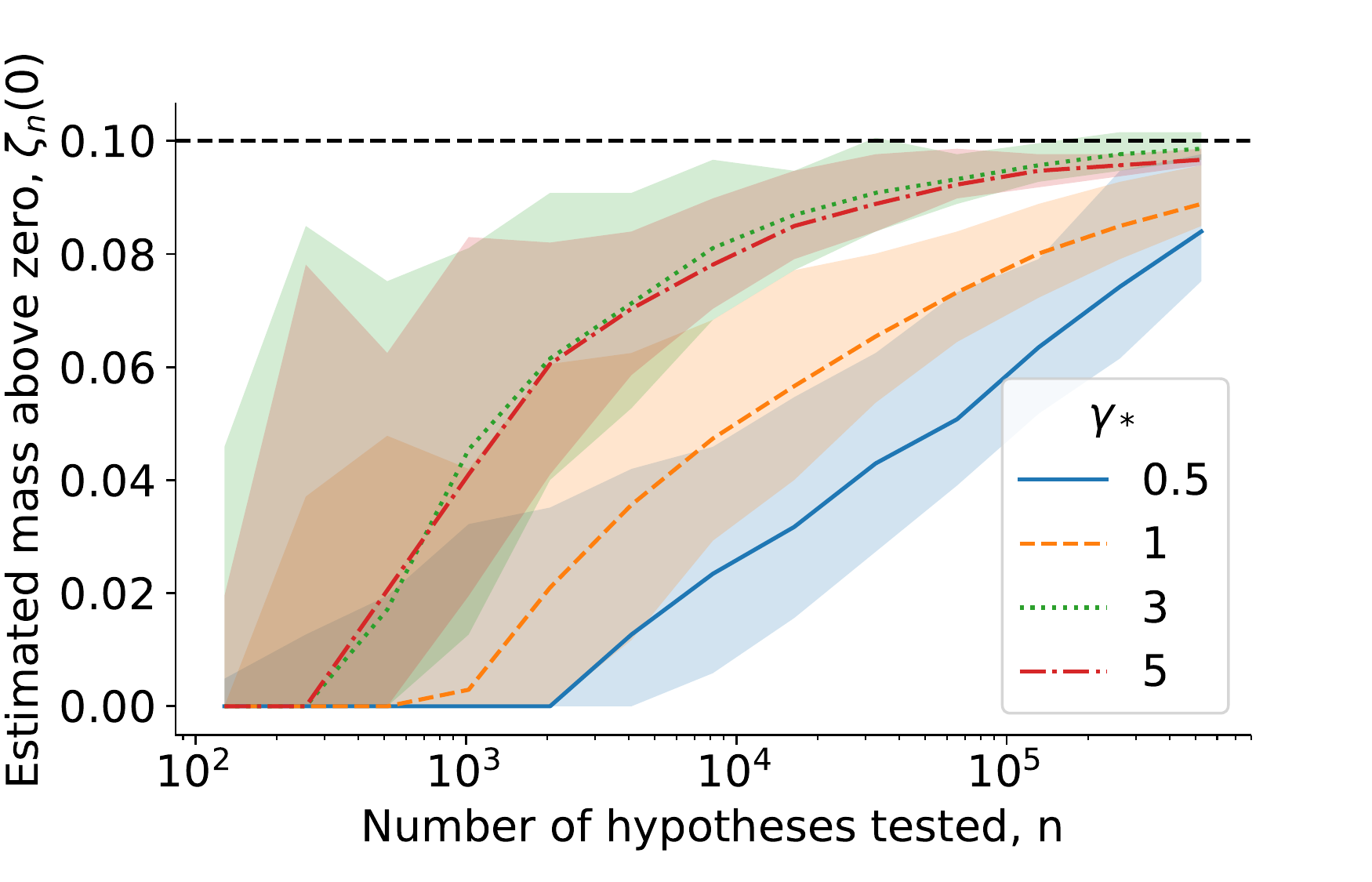}
    \caption{
    {
    \ifx\figCapSmall\undef
    \else
    \small 
    \fi
    Median and 90\% bootstrapped confidence intervals for $\widehat{\zeta}_n(0)$, where $\nu_*=(1-\zeta_*)\delta_0 + \zeta_*\delta_{\gamma_*}$, for various $\gamma_*$. As $n$ increases, for a constant $\zeta_*=0.1$, our estimator $\widehat{\zeta}_n(0)$ converges to $\zeta_*$ without overestimating. As expected, the estimates are lower (have more error) when the alternate effect size $\gamma_*$ is small.}}
    \label{fig:zetaHatConsistency}
\end{figure}
\begin{figure}
    \centering 
    \includegraphics[trim={4cm 13.5cm 0.1cm 15cm},clip,width=\columnwidth]{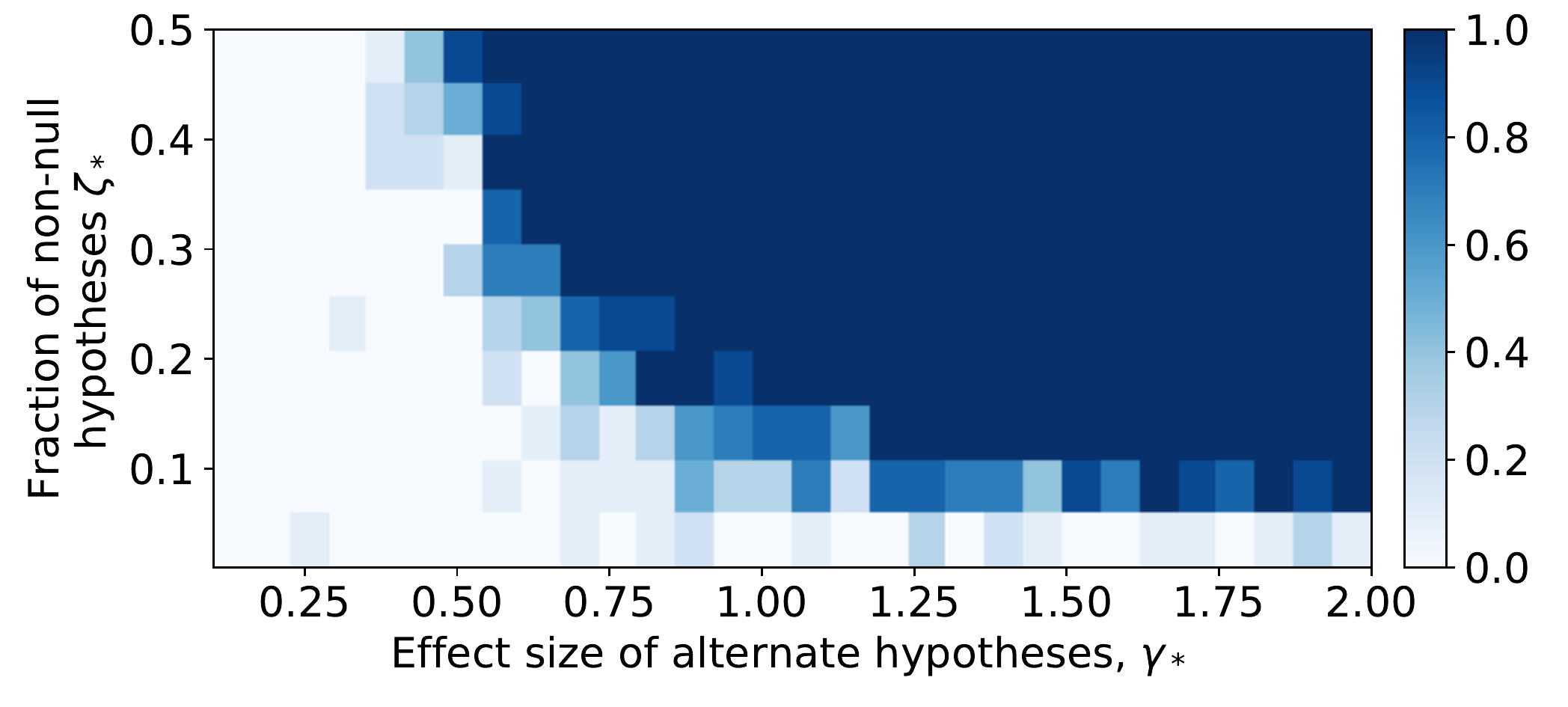}
    \caption{
    Empirical $\P(\widehat{\zeta}_n(0) > \tfrac{1}{2}\zeta_*)$, for various parameterizations $(\gamma_*, \zeta_*)$ of the two-spike Gaussian model \eqref{eq:canonicalTwoSpikes}. For a fixed value of $n=10^4$, the probability of detecting at least half of the discoveries increases 
    with both $\gamma_*$ and $\zeta_*$.
    Empirical probabilities were computed over ten trials.}
    \label{fig:heatMap}
\end{figure}
\begin{figure*}
    \centering
    \includegraphics[width=0.9\textwidth]{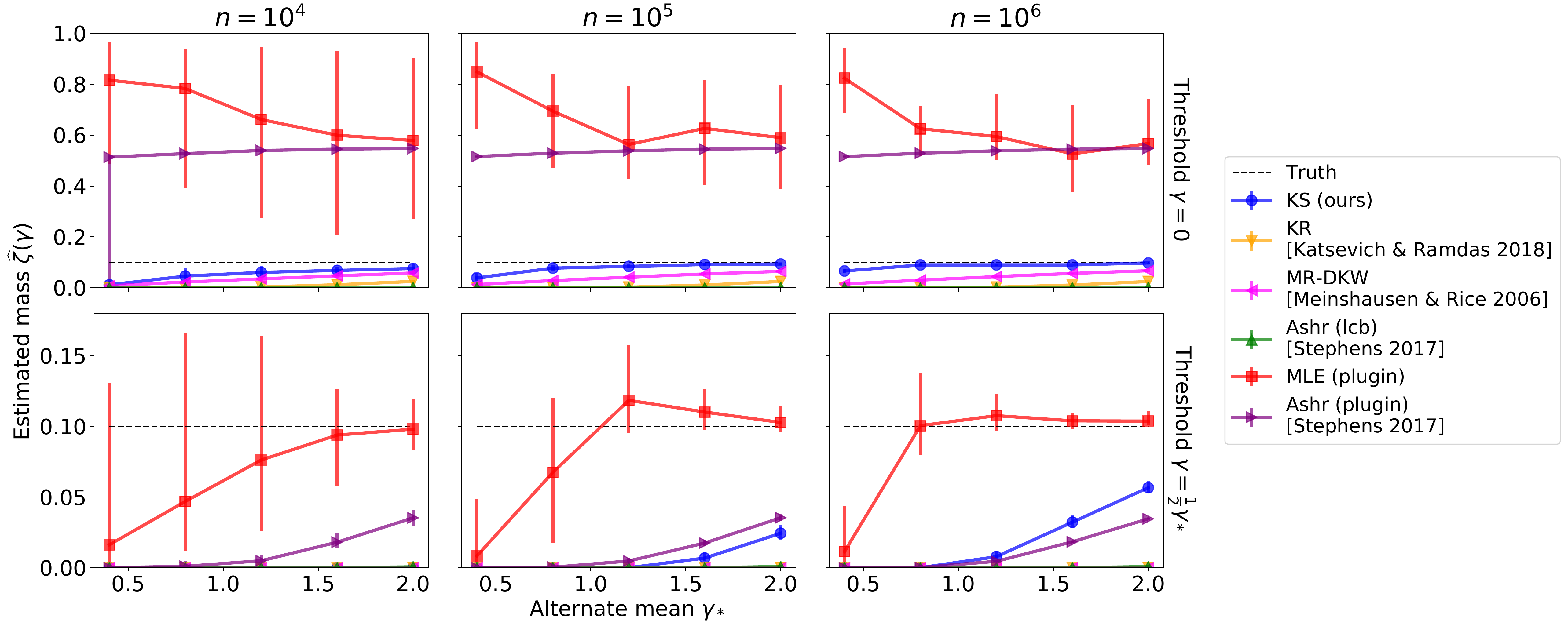}
    \caption{Our estimator outperforms the baselines in the mixture of two Gaussians setting, Eqn \eqref{eq:canonicalTwoSpikes}, returning $\widehat{\zeta}_n(\gamma)$ close to the truth without overestimating, for a wide variety of $n$ and $\gamma$.}
    \label{fig:baselines}
\end{figure*}

\section{Experiments}\label{sec:experiments}
Details of our implementation can be found in Appendix \ref{app:implementation}. A Python implementation is available at \url{https://github.com/jenniferbrennan/CountingDiscoveries/}.

\subsection{Experimental Results on Simulated Data}
We evaluate our estimator on both real and simulated data. We begin with the mixture of two Gaussians 
described by Eqn \eqref{eq:canonicalTwoSpikes}.
Figure \ref{fig:zetaHatConsistency} shows the rate of convergence of our estimator for different values of $\gamma_*$, the alternate effect size. Note that the estimate never exceeds the true value $\zeta_*$, and that it improves as $n$ increases. The variance of our estimator, shown with bootstrapped 90\% confidence intervals, can be large for small $n$ but decreases as $n$ increases.

For a fixed value of $n$, we are interested in the probability that our estimator detects at least half of the discoveries, $
\P_{\nu_*}(\widehat{\zeta}_n(0) \geq \tfrac{1}{2}\zeta_*)$,
as a function of the fraction of discoveries $\zeta_*$ and their effect size $\gamma_*$. Our estimator exhibits a sharp transition between detecting fewer than half and detecting more than half, as shown in Figure \ref{fig:heatMap}.

We compare our estimator to several baselines found in the literature. Figure \ref{fig:baselines} shows the performance of various estimators in the mixture of two Gaussians setting. Each of the six panels shows how the estimators perform as the alternate mean $\gamma_*$ varies, for different numbers of hypotheses $n$ and tested thresholds $\gamma$. We see that the two plugin methods, the MLE and the plugin \texttt{ashr} estimate \cite{stephens2017false}, both fail to satisfy our constraint $\widehat{\zeta}(\gamma) \leq \zeta_*(\gamma)$. Among the four remaining estimators, ours comes closest to the truth. Notably, our estimator continues to improve as $n$ increases, whereas the baselines do not. We conclude that our estimator outperforms existing methods in the regime of large $n$ and $\gamma_* \approx \sigma$. Appendix \ref{app:baselines} includes a plot comparing only the four estimators that satisfy our constraint.

We also demonstrate that our method works on distributions other than Gaussians by applying it to synthetic Poisson and binomial data (Figure \ref{fig:poi_bin_exp}). Details of the experiments can be found in Appendix \ref{app:implementation}; we note that our test detected the presence of the alternate hypotheses even when no alternates were identifiable via Bonferroni-corrected multiple testing.

\begin{figure}
    \centering 
    \includegraphics[trim={1.5cm 5.6cm 18.5cm 3cm},clip, width=0.9\columnwidth]{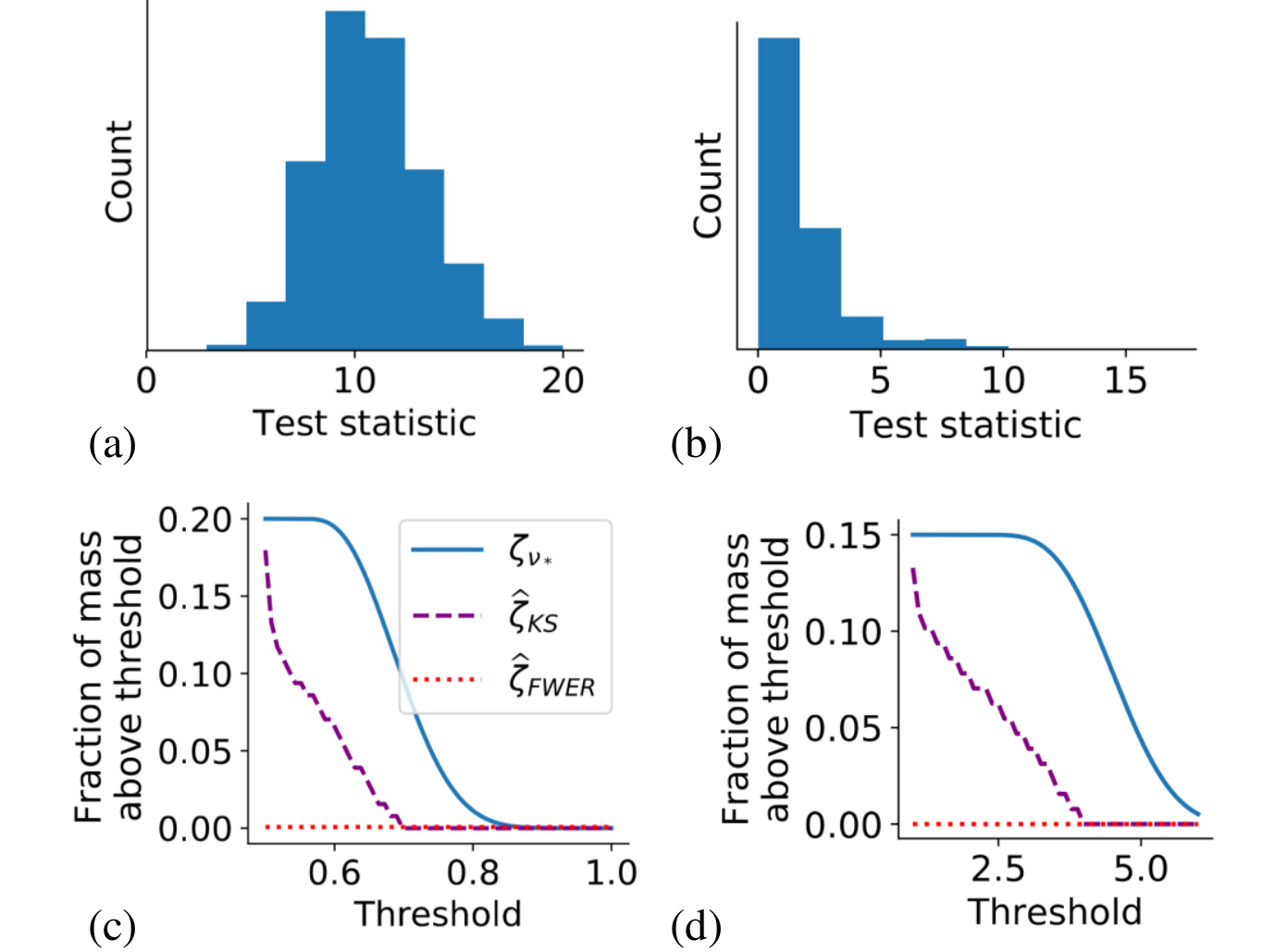}
    \caption{{
    Performance of our estimator on binomial (left) and Poisson (right) data. Top panels show the observed histogram of $X_i$. Bottom panels show the true fraction of effects above each threshold ($\zeta_{\nu_*}$), as well as estimates using our method ($\widehat{\zeta}_{KS}$) and identification via Bonferroni-corrected multiple testing ($\widehat{\zeta}_{FWER}$). Our estimator gets closer to the truth, without overestimating.
    }}
    \label{fig:poi_bin_exp}
\end{figure}

\subsection{Experimental Results on Real Data}
We evaluated our estimator on Z-scores from an experiment to identify which genes contribute to influenza replication in \textit{Drosophila}, described by Hao et al. \yrcite{hao2008drosophila}. The data, available in our supplementary material, consisted of Z-scores from two replicates for each of 13,071 genes. Figure \ref{fig:fly}(a) shows the empirical distribution of the 13,071 averaged Z-scores, which are the observations $X_i$. The theoretically motivated distribution $X_i \sim \mathcal{N}(\mu, \tfrac{1}{2})$ is a poor fit to this data, perhaps due to undocumented pre-processing steps not annotated in the dataset, so we began by estimating the variance of these observations. We found that $\sigma^2 = \tfrac{1}{4}$ provided a good fit to the data; we used this value for the rest of our computations. Testing for significant effect sizes using Bonferroni correction at the 0.05 level (critical value shown in Figure \ref{fig:fly}(a)) resulted in 83 discoveries, representing 0.6\% of genes. Given the low number of replicates performed in this experiment, we might suspect that there are more discoveries with smaller effect sizes.

Figure \ref{fig:fly}(b) shows the results of the plug-in MLE estimator $\widehat{\zeta}_{MLE}$, our estimator ($\widehat{\zeta}_{KS}$), and identification with Bonferroni correction ($\widehat{\zeta}_{FWER}$). The fitted MLE suggests that there are around 2600 discoveries to be made (20\% of genes), with most effect sizes around 1. As discussed previously, the MLE can overestimate the true number of discoveries and their effect sizes. Our conservative estimator guarantees that there are at least 1400 genes (11\% of all genes) with positive effects, including at least 190 genes (1.5\%) with effect size of at least 0.5. 
Our estimator generally detected more discoveries than $\widehat{\zeta}_{FWER}$, excluding the influence of the 23 genes (0.2\%) with $X_i > 3$. These observations fall into the \textit{sparse regime} \cite{donoho2004}, where our estimator has less power. 
These results could facilitate the design of an experiment to identify genes with effect sizes exceeding some threshold, or upper bound the cost of a sequential experiment to identify the top 200 genes.
\begin{figure}
    \centering 
    \includegraphics[trim={0cm 14cm 22cm 0cm},clip,width=0.9\columnwidth]{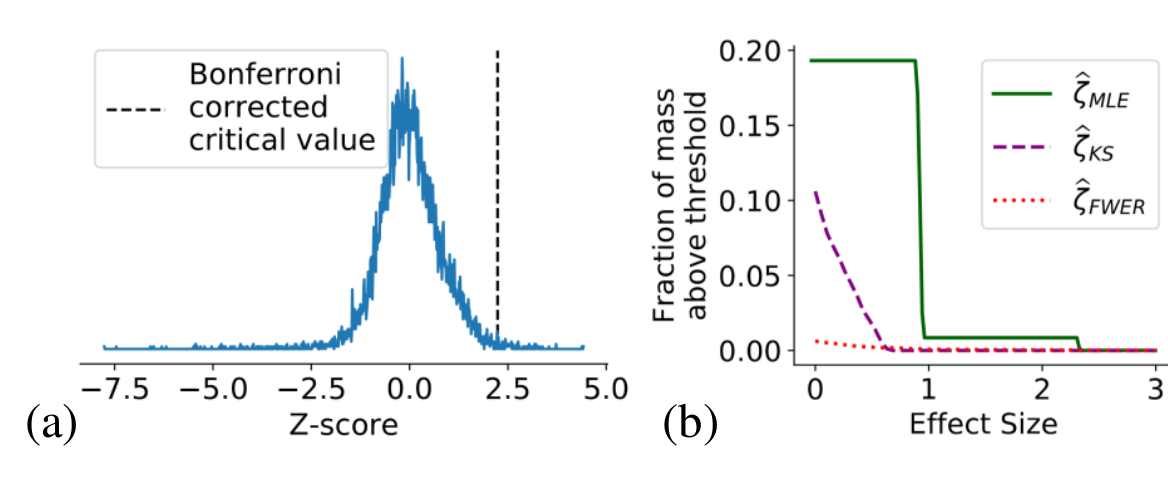}
    \caption{
    Two Z-scores were averaged for each of 13,071 \textit{Drosophila} genes. Even though (a) indicates that very few discoveries could be made, (b) shows that the MLE suggests, and our estimator confirms, many discoveries exist. We note that the MLE provides no guarantee of a conservative estimate, and may drastically overestimate the true fraction at any point.}
    \label{fig:fly}
\end{figure}

\section{Discussion and Future Work}\label{sec:future}
We have presented an algorithm that estimates the fraction of a mixing distribution that lies above some threshold, subject to the constraint that the estimate does not exceed the true fraction with high probability. Our algorithm can be generalized to the following template:
\begin{enumerate}
    \item Choose some distance metric on CDFs.
    \item Find the set $\mathcal{A}$ of ``plausible'' $F_{\nu}$ given observation $\widehat{F}_n$, which are the CDFs such that $d(F_{\nu}, \widehat{F}_n) < \tau_{\alpha, n}$.
    \item Choose $\tau_{\alpha, n}$ such that $\P(d(\widehat{F}_n, F_{\nu_*}) > \tau_{\alpha, n})\leq \alpha$. 
\end{enumerate}
Returning the minimum amount of mass above the threshold, over the set of plausible distributions $\mathcal{A}$, guarantees with high probability that we do not overestimate the true mass. Our algorithm instantiates this template with the $\ell_\infty$ norm as the distance metric. 

Another natural choice of metric is the likelihood of $\widehat{F}_n$ given $F_\nu$. In order to use this in our template, we need finite sample bounds on the likelihood of $\widehat{F}_n$ given $F_{\nu_*}$. Asymptotic versions of these results are worked out by Jiang \& Zhang \yrcite{glrt} for the case of Gaussian $X_i$, and it would be easy to extend these to finite sample bounds. Extensions of Jiang \& Zhang's results should show that the resulting estimator is optimal throughout the so-called \textit{sparse} and \textit{dense} regimes \cite{donoho2004}. Unfortunately, their value of $\tau_{\alpha, n}$ depends on unknown constants, and therefore it would require extensive simulations with thousands of repetitions for each $(\zeta, \gamma)$ pair to obtain a reliable estimate of the critical values for pilot study analysis. In addition, using the likelihood-based approach for a new distribution $f_\mu$ requires an entirely new proof of the high-probability bound on the likelihood ratio.

We believe it would be possible to modify our estimator in order to improve its performance in the sparse regime, where effects are large but rare. Our estimator uses the DKW inequality \cite{massart1990tight} to measure the plausibility of a latent distribution $\nu$, but the DKW inequality is not tight where the empirical CDF has low variance. Such points occur in the sparse regime, for example at $F_{\nu_*}(\tfrac{1}{2}\gamma_*)$. Applying a bound that uses variance information, such as an empirical Bernstein DKW (as surveyed in, e.g., \cite{howard2019sequential}), could address this lack of power.

Any algorithm for this problem necessarily makes some assumptions about the data generating process, otherwise all observations could come from the alternate, with $X\sim P_1$ having density $\widehat{F}_n$. As discussed in Section \ref{sec:relatedworks}, previous works have used various assumptions, such as unimodality of $\nu_*$ or ``purity'' of p-values around zero.
Our key assumption is the parametric form of $X_i$, under both the null and the alternate.
In practice, in order to decrease our reliance on this assumption, we could learn some parameter of the test statistic distribution from the observations themselves. This was our approach with the \textit{Drosophila} data, when we fit the variance $\sigma^2$ of the Z-score. This approach is also taken by Efron \yrcite{efron2007size}. Even more ambitiously, we could learn the conditional distribution $f(X|\mu)$ by fitting it jointly with the means $\mu$, and then use this conditional distribution to generate $F_\nu$ from a candidate distribution $\nu$. 
\newpage
\section*{Acknowledgements}
We thank Lalit Jain and Weihao Kong for their insightful comments, and we thank Swati Padmanabhan for valuable discussions about convex optimization. Jennifer Brennan is supported by an NSF Graduate Research Fellowship.

{
\bibliography{Bibliography}
\bibliographystyle{icml2020}
}

\appendix
\onecolumn

\section{Proofs of testing results for mixtures of two Gaussians}
\subsection{Useful Lemmas}
In this section of the appendix, we will provide proofs of the results for hypothesis testing - i.e., determining whether there is any mass above $0$ - in the two-spike Gaussian setting. This is the simplest setting we consider.

In order to extend our lower bound in Theorem \ref{lem:est-finiteSample} to the case of two Gaussian spikes, we need to compute the quantity
\begin{align*}
    \min_{\nu \in S(0,0)} ||F_\nu - F_{\nu_*}||_\infty
\end{align*}
for this setting.

\begin{lemma}\label{lem:sampleComplexityDetection}
Let $\nu_* = (1-\zeta_*)\delta_0 + \zeta_*\delta_{\gamma_*}$, and let $S$ be defined as in Eqn (\ref{eqn:defineS}). Then,
\begin{align*}
    \min_{\nu\in S(0,0)}||F_\nu - F_{\nu_*}||_\infty
    &\geq \zeta_*\left( \Phi_\sigma(\tfrac{1}{2}\gamma_*) - \Phi_\sigma(-\tfrac{1}{2}\gamma_*) \right)
\end{align*}
where $\Phi_\sigma$ is the CDF of the normal distribution $\mathcal{N}(0, \sigma^2)$. Furthermore, if $\gamma_* < \sigma$, then
\begin{align*}
    \min_{\nu\in S(0,0)}||F_\nu - F_{\nu_*}||_\infty
    &\geq\frac{23\zeta_*\gamma_*}{24\sigma\sqrt{2\pi}}
\end{align*}
\end{lemma}
\begin{proof}
First, we lower bound this minimax problem by a minimization at a specific point, $t=\tfrac{1}{2}\gamma_*$,
\begin{align*}
    \min_{\nu\in S(0,0)} ||F_\nu - F_{\nu_*}||_\infty
    &= \min_{\nu\in S(0,0)} \sup_{t\in\mathbb{R}} |F_\nu(t) - F_{\nu_*}(t)|\\
    &\geq \min_{\nu\in S(0,0)} |F_\nu\left(\tfrac{1}{2}\gamma_*\right) - F_{\nu_*}\left(\tfrac{1}{2}\gamma_*\right)|\\
    &\geq \min_{\nu\in S(0,0)} F_\nu\left(\tfrac{1}{2}\gamma_*\right) - F_{\nu_*}\left(\tfrac{1}{2}\gamma_*\right)
\end{align*}
Next, we lower bound $F_\nu(\tfrac{1}{2}\gamma_*)$. By definition,
\begin{align*}
    F_\nu (\tfrac{1}{2}\gamma_*) 
    &= \int_{-\infty}^\infty \nu(x) \Phi_\sigma(\tfrac{1}{2}\gamma_*-x)dx\\
    &\geq \Phi(\tfrac{1}{2}\gamma_*)
\end{align*}
where the inequality follows by the constraint $\nu\in S(0,0)$, so $\nu$ can have no mass above $0$. We use this bound, as well as the exact value of $F_{\nu_*}\left(\tfrac{1}{2}\gamma_*\right)$, to lower bound the minimum,
\begin{align*}
    \min_{\nu\in S(0,0)} ||F_\nu - F_{\nu_*}||_\infty
    &\geq \Phi_\sigma(\tfrac{1}{2}\gamma_*) - (1-\zeta_*)\Phi_\sigma(\tfrac{1}{2}\gamma_*) - \zeta_*\Phi_\sigma(-\tfrac{1}{2}\gamma_*)\\
    &= \zeta_*\left(\Phi_\sigma(\tfrac{1}{2}\gamma_*) - \Phi_\sigma(-\tfrac{1}{2}\gamma_*) \right)
\end{align*}
which proves the first claim. For the second claim, when $\gamma_* < \sigma$, we use the second order Taylor series approximation to the Gaussian density. We have
\begin{align}
    \Phi_\sigma(\tfrac{1}{2}\gamma_*) - \Phi_\sigma(-\tfrac{1}{2}\gamma_*) &= 2\P(0 \leq X \leq \tfrac{1}{2}\gamma_*)\label{eqn:prob-aux}
\end{align}
for $X\sim\mathcal{N}(0,\sigma^2)$. Taking a quadratic approximation to the normal density yields
\begin{align*}
    f(x) &= \frac{1}{\sigma\sqrt{2\pi}} \left(1 + \frac{1}{2\sigma^2}e^{-c^2/2\sigma^2}\left(\frac{c^2}{\sigma^2} - 1\right) x^2\right)
\end{align*}
for some $c\in[0, \gamma_*]$. This is minimized for $c=0$, which gives us
\begin{align*}
    f(x) \geq \frac{1}{\sigma\sqrt{2\pi}} \left(1 - \frac{1}{2\sigma^2}x^2\right)
\end{align*}
We can now lower bound (\ref{eqn:prob-aux}) as follows,
\begin{align*}
    \P(0 \leq X \leq \tfrac{1}{2}\gamma_*) 
    &=  \frac{1}{\sigma\sqrt{2\pi}} \int_0^{\tfrac{1}{2}\gamma_*}  e^{-x^2/2\sigma^2} dx\\
    &\geq \frac{1}{\sigma\sqrt{2\pi}} \int_0^{\tfrac{1}{2}\gamma_*}  (1 - \frac{1}{2\sigma^2}x^2) dx\\
    &= \frac{1}{\sigma\sqrt{2\pi}} (x - \tfrac{1}{6\sigma^2}x^3)|_{x=0}^{x=\tfrac{1}{2}\gamma_*}\\
    &= \frac{1}{\sigma\sqrt{2\pi}}(\tfrac{1}{2}\gamma_* - \tfrac{1}{48\sigma^2} \gamma_*^3)\\
    &= \frac{1}{2\sigma\sqrt{2\pi}} \gamma_*\left(1 - \frac{\gamma_*^2}{24\sigma^2}\right)
\end{align*}
Since $\gamma_* < \sigma$, this is always positive. We can also use that fact to bound this by
\begin{align*}
    \P(0 \leq X \leq \tfrac{1}{2}\gamma_*) 
    &\geq \frac{1}{2\sigma\sqrt{2\pi}} \gamma_* \left(\frac{23}{24}\right)
\end{align*}
so that
\begin{align*}
    \Phi_\sigma(\tfrac{1}{2}\gamma_*) - \Phi_\sigma(-\tfrac{1}{2}\gamma_*) 
    \geq
    \frac{23}{24\sigma\sqrt{2\pi}} \gamma_*
\end{align*}
which completes the proof.
\end{proof}

Next, we state and prove another lemma that will be useful later. This result bounds the probability that our estimator $\widehat{\zeta}(0)$ returns $0$. This is related to the probability of detecting the existence of alternate hypotheses.
\begin{lemma}\label{lem:type2errLowerBoundZetaZero}
Let $\nu_* = (1-\zeta_*)\delta_0 + \zeta_*\delta_{\gamma_*}$, and let $\widehat{\zeta}(0)$ be our estimator from Eqn (\ref{eqn:estimator}) evaluated at $0$. Then
\begin{align}
    \P(\widehat{\zeta}(0) = 0) &\leq 2\exp\left( -2n \left(\zeta_*\left(\Phi(\tfrac{1}{2}\gamma_*)-\Phi(-\tfrac{1}{2}\gamma_*)\right) - \tau_{\alpha, n}\right)^2 \right)
\end{align}
\end{lemma}
\begin{proof}
We begin by substituting in the definition of our estimator. Recall the definition of $S$, from Eqn \eqref{eqn:defineS}.
\begin{align}
    \mathbb{P}\left( \widehat{\zeta}(0) = 0 \right) &= \mathbb{P}\left( \max\left( \zeta \geq 0 : \min_{\nu\in S(0, 0)} ||\widehat{F}_n - F_\nu ||_\infty > \tau_{\alpha, n}\right) = 0\right)\\
    &= \mathbb{P}\left( \min_{\nu\in S(0, 0)} ||\widehat{F}_n - F_\nu ||_\infty  \leq \tau_{\alpha, n}\right)\\
    &\leq \mathbb{P}\left( \min_{\nu\in S(0, 0)} ||F_\nu - F_{\nu_*} ||_\infty - ||\widehat{F}_n - F_{\nu_*}||_\infty \leq \tau_{\alpha, n}\right)\\
    &=  \mathbb{P}\left( ||\widehat{F}_n - F_{\nu_*}||_\infty \geq \min_{\nu\in S(0, 0)} ||F_\nu - F_{\nu_*} ||_\infty - \tau_{\alpha, n}\right)
\end{align}
Next, we apply the DKW inequality, along with Lemma \ref{lem:sampleComplexityDetection}
\begin{align}
    \mathbb{P}\left( \widehat{\zeta}(0) = 0 \right) &\leq 2\exp\left( -2n \left(\min_{\nu\in S(0, 0)} ||F_\nu - F_{\nu_*} ||_\infty - \tau_{\alpha, n}\right)^2 \right)\\
    &= 2\exp\left( -2n \left(\zeta_*\left(\Phi(\tfrac{1}{2}\gamma_*)-\Phi(-\tfrac{1}{2}\gamma_*)\right) - \tau_{\alpha, n}\right)^2 \right)
\end{align}
This concludes the proof.
\end{proof}

\subsection{Proof of Corollary \ref{cor:est-finiteSampleGauss}}
\begin{proof} \textbf{Corollary \ref{cor:est-finiteSampleGauss}} This corollary follows immediately from Theorem \ref{lem:est-finiteSample} with $\alpha = \delta$, $\varepsilon = \zeta_{\nu_*}(0) = \zeta_*$ and $\min_{\nu\in S(0,0)}||F_\nu - F_{\nu_*}||_\infty$ bounded as in Lemma \ref{lem:sampleComplexityDetection}.
\end{proof}

\subsection{Comparison to finite-sample lower bounds}

In this section, we prove the statements in Section \ref{sec:globalNull}.

To show that the sample complexity for our hypothesis test matches the lower bound of \cite{MostBiasedCoin} up to constants, it is helpful to write our sample complexity as a maximum of two quantities.


\begin{lemma}\label{lem:MBC-algebra}
Let $\delta \leq 1$, $\zeta_* \leq 1$, and $\sigma^2 > \gamma_*^2$. Then
\begin{align*}
    \frac{16\sigma^2\log\left(\frac{2}{\delta}\right)}{\zeta_*^2\gamma_*^2}
    &= \max\left\{ \frac{1}{\zeta_*}, \frac{16\sigma^2\log\left(\frac{2}{\delta}\right)}{\zeta_*^2\gamma_*^2}\right\}
\end{align*}
\end{lemma}
\begin{proof}
We will prove this by showing that the first term of the max is always smaller than the second. We have
\begin{align*}
    \frac{16\sigma^2\log\left(\frac{2}{\delta}\right)}{\zeta_*^2\gamma_*^2}
    &> \frac{16\log\left(\frac{2}{\delta}\right)}{\zeta_*^2}\\
    &\geq \frac{16\log\left(2\right)}{\zeta_*^2}\\
    &\geq \frac{1}{\zeta_*^2}\\
    &\geq \frac{1}{\zeta_*}
\end{align*}
which concludes the proof.
\end{proof}
We see that our test matches the lower bound given in \cite{MostBiasedCoin} up to constants, as long as $\delta$ is bounded away from $1$ (so that the first term in the lower bound cannot be arbitrarily small).

\subsection{Proof of Proposition \ref{prop:prescreen}}
\begin{proof} \textbf{Proposition \ref{prop:prescreen}}
Recall that $m$ denotes the number of genes we will test.
By Corollary \ref{cor:est-finiteSampleGauss} with $\alpha=\delta$, we have that
\begin{align*}
    m \geq \frac{\log(\frac{2}{\delta})}{\zeta_*^2 (\Phi_\sigma(\frac{1}{2}\gamma_*) - \Phi_\sigma(-\frac{1}{2}\gamma_*))^2}
\end{align*}
implies
\begin{align*}
    \P(\zeta_*-\widehat{\zeta}_m > 0) < \delta.
\end{align*}
We consider two cases. First, if $\gamma_* < \sigma$, then Corollary \ref{cor:est-finiteSampleGauss} states that the sample size must be at least
\begin{align*}
    m \geq  \frac{16\sigma^2\log(\frac{2}{\delta})}{\varepsilon^2\gamma_*^2}
\end{align*}
to guarantee $\widehat{\zeta}(0) > 0$ with high probability. We can use the relationship $\sigma^2 = \tfrac{1}{t}$ to get the first requirement of the lemma.
\begin{align*}
    \gamma_* &\geq \sqrt{\frac{16\log(\frac{2}{\delta})}{\zeta_*^2 m t}}\\
    &= \sqrt{\frac{16\log(\frac{2}{\delta})}{\zeta_*^2 B}}
\end{align*}

The second requirement comes from the case of large $\gamma_*$. When $\gamma_*\geq \sigma$, we can use a table of values of $\Phi$ to find
\begin{align*}
    \Phi_\sigma(\tfrac{1}{2}\gamma_*) - \Phi_\sigma(-\tfrac{1}{2}\gamma_*)
    &= 1 - 2\Phi_\sigma(-\tfrac{1}{2}\gamma_*)\\
    &\geq 1 - 2\Phi_\sigma(-\tfrac{1}{2}\sigma)\\
    &= 1 - 2\Phi_1(-\tfrac{1}{2})\\
    &\geq \frac{3}{10}
\end{align*}
We see that in this case, as long as 
\begin{align}
    m \geq  \frac{4\log(\frac{2}{\delta})}{\zeta_*^2} > \frac{200\log(\frac{2}{\delta})}{81\zeta_*^2}
\end{align}
(from Corollary \ref{cor:est-finiteSampleGauss}), then we have $\widehat{\zeta}(0) > 0$ with high probability. If both of these requirements hold, then we have $\widehat{\zeta}_n(0) > 0$ with high probability, regardless of the value of $\gamma_*$.
\end{proof}

Finally, we prove the statements in Remark \ref{remark:nTimesBetter}: that identifying these alternate hypotheses requires at least $nB$ total replicates (while our estimator only takes $B$ total replicates to count them) as long as the budget satisfies $B = O(n)$, and the fraction of alternates is constant.

\begin{proof}\textbf{Remark \ref{remark:nTimesBetter}}
Identifying the alternates takes at least order $n\gamma_*^{-2}\log(1/\delta)$, even without correcting for multiple testing, because we require at least order $\gamma_*^{-2}\log(1/\delta)$ samples to estimate the mean of a Gaussian within additive error $\gamma_*$, with probability $1-\delta$.

Our test, by contrast, takes the larger of 
\begin{align*}
    B = O\left( \max\left\{ \log(1/\delta)\zeta_*^{-2}\gamma_*^{-2}, t\gamma_*^{-2} \right\} \right).
\end{align*}
If $\zeta_*$ is constant, and if $t$ is constant (which occurs when $B = O(n)$ and we allocate replicates equally across all $n$ hypotheses), then our estimator can count half of the alternates using
\begin{align*}
    B = O\left( \log(1/\delta)\gamma_*^{-2} \right),
\end{align*}
which is $n$ times fewer than the number of samples required to identify the discoveries.
\end{proof}

\section{Proofs of Corollary \ref{cor:estimation-finiteSampleGauss} and Lemma \ref{lem:finiteSampleEstBound} (Estimation results for mixtures of two Gaussians)}
In this section, we provide proofs of the upper and lower bounds for estimating the amount of mass above the threshold $0$ when the data is drawn from a mixture of two Gaussians.

\subsection{Estimation upper bound (Corollary \ref{cor:estimation-finiteSampleGauss})}

Corollary \ref{cor:estimation-finiteSampleGauss} is a consequence of Lemma \ref{lem:boundingLInfForEstBound}, which bounds the minimum $\ell_\infty$ distance between  $F_{\nu_*}$ and any CDF with less than $\tfrac{1}{2}\zeta_*$ mass above zero, and Theorem \ref{lem:est-finiteSample}, which relates this quantity to the sample complexity of estimating the true amount of mass with accuracy $\varepsilon = \tfrac{1}{2}\zeta_*$. We begin by stating this key lemma, and then we establish a series of technical lemmas necessary for the proof.

The key proof idea is that the quantity bounded in Lemma \ref{lem:boundingLInfForEstBound} is the solution to a convex optimization problem, and we can identify the optimal point:
\begin{align}
    \nu_{OPT} &= \left(1-\tfrac{1}{2}\zeta_*\right)\delta_0 + \tfrac{1}{2}\zeta_*\delta_{2\gamma_*}. \label{eqn:optimalNu-estimation}
\end{align}
In this section, we follow the conventions of Boyd \& Vandenberghe \yrcite{boyd2004convex} (p 127) to describe the optimal value and optimal point of an optimization problem. In particular, the \textit{optimal value} of a minimization problem is the minimum value of the objective function over the constraint set, while an \textit{optimal point} is a point $\nu_{OPT}$ in the constraint set that achieves the optimal value. We depart from the notation of Boyd \& Vandenberghe in one significant respect: we denote the optimal point by $\nu_{OPT}$, and not $\nu_*$ (as we will be using the subscript $*$ to denote the reference distribution in our optimization problem).

The first task is to establish the optimality of this solution. We proceed using the standard machinery of convex optimization, as described by Boyd \& Vandenberghe \yrcite{boyd2004convex}: our problem satisfies strong duality, so we exhibit a dual feasible solution with corresponding primal point given by \eqref{eqn:optimalNu-estimation}. Once we have established optimality of this guess, we use several Taylor series approximations to bound the optimal value by a polynomial in $\gamma_*$ and $\zeta_*$. After establishing this series of technical lemmas, we will prove Lemma \ref{lem:boundingLInfForEstBound}.

\begin{lemma}\label{lem:boundingLInfForEstBound}
Let $X_i \sim \mathcal{N}(\mu_i, 1)$ and $\mu_i \sim \nu_*$, with $\nu_* = (1-\zeta_*)\delta_0 + \zeta_*\delta_{\gamma_*}$. Let $\zeta_* > 0$ and $\gamma_*\in (0, 1]$. Then
\begin{align}
    \min_{\nu : \P(\mu > 0) < \frac{1}{2}\zeta_*} ||F_\nu - F_{\nu_*}||_\infty \geq 0.01 \gamma_*^2 \zeta_* \label{eqn:estimation-optimizationProblem}
\end{align}
\end{lemma}

In order to show that our solution \eqref{eqn:optimalNu-estimation} to the optimization problem \eqref{eqn:estimation-optimizationProblem} is optimal, we need to demonstrate certain properties of the subgradient of the objective function. In order to analyze the subgradient of an $\ell_\infty$ norm, it is first necessary to characterize the value(s) of $t$ for which the supremum in this sup-norm is attained. Our first technical lemma computes these maximizing values of $t$ when $\nu$ is our conjectured optimal value, $\nu_{OPT}$.
\begin{lemma}\label{lem:tPlusAndMinus}
Let $\gamma_* \in (0, 1]$ and $\zeta_* \in (0,1]$. Define
\begin{align*}
    f(t)
    &:= \left| \frac{1}{\sqrt{2\pi}}\int_{x=-\infty}^t (1-\tfrac{1}{2}\zeta_*) e^{-x^2/2} + \tfrac{1}{2}\zeta_* e^{-(x-2\gamma_*)^2/2} - (1-\zeta_*)e^{-x^2/2} - \zeta_* e^{-(x-\gamma_*)^2/2} dx \right|\\
    &:= \left| \tilde{f}(t)\right|
\end{align*}
Then
\begin{align*}
    \arg\sup_t f(t) &= \{ t_+, t_- \} 
\end{align*}
where
\begin{align*}
    t_+ &:= \tfrac{3}{2}\gamma_* + \tfrac{1}{\gamma_*}\log\left(1-\sqrt{1-e^{-\gamma_*^2}}\right)\\
    t_- &:= \tfrac{3}{2}\gamma_* + \tfrac{1}{\gamma_*}\log\left(1+\sqrt{1-e^{-\gamma_*^2}}\right)
\end{align*}
with $\text{sign}(\tilde{f}(t_+)) = 1$ and $\text{sign}(\tilde{f}(t_-)) = -1$.
\end{lemma}
\begin{proof}
We begin by arguing that we only need to consider the points at which $\tilde{f}'(t) = 0$. To find the argsup, we examine the critical points of $f(t)$. Note that the critical points of $f(t)$ include all critical points of $\tilde{f}(t)$, as well as points at which $f(t)=0$. Since we are interested in finding the supremum of $f(t)$, and since $f(t) > 0$ for some $t$ (because the argument to the integral is not identically zero), we can discard any critical point at $f(t) = 0$. We conclude that it suffices to examine the critical points of $\tilde{f}(t)$.

We begin by noting that $\lim_{t\to\infty} \tilde{f}(t) = \lim_{t\to -\infty} \tilde{f}(t) = 0$, so the supremum is not found at extreme values of $t$. This means we only need to inspect the values of $t$ where the derivative $\tilde{f}'(t) = 0$. We compute this derivative,
\begin{align*}
    \tilde{f}'(t) &= \frac{d}{dt} \frac{1}{\sqrt{2\pi}}\int_{x=-\infty}^t (1-\tfrac{1}{2}\zeta_*) e^{-x^2/2} + \tfrac{1}{2}\zeta_* e^{-(x-2\gamma_*)^2} - (1-\zeta_*)e^{-x^2/2} - \zeta_* e^{-(x-\gamma_*)^2/2} dx\\
    &= (1-\tfrac{1}{2}\zeta_*) \tfrac{1}{\sqrt{2\pi}} e^{-t^2/2} + \tfrac{1}{2}\zeta_* \tfrac{1}{\sqrt{2\pi}} e^{-(t-2\gamma_*)^2} - (1-\zeta_*)\tfrac{1}{\sqrt{2\pi}} e^{-t^2/2} - \zeta_* \tfrac{1}{\sqrt{2\pi}} e^{-(t-\gamma_*)^2/2}\\
    &=  \tfrac{1}{2}\zeta_* \tfrac{1}{\sqrt{2\pi}} e^{-t^2/2} +\tfrac{1}{2}\zeta_* \tfrac{1}{\sqrt{2\pi}} e^{-(t-2\gamma_*)^2/2} - \zeta_* \tfrac{1}{\sqrt{2\pi}} e^{-(t-\gamma_*)^2/2}
\end{align*}
by the fundamental theorem of calculus. Next, we set the derivative to zero. Due to the specific coefficients in the definition of $\widetilde{f}(t)$, the derivative is quadratic in $t$ and can be solved exactly,
\begin{align*}
    0 
    &=  -\tfrac{1}{2}\zeta_* \tfrac{1}{\sqrt{2\pi}} e^{-t^2/2} -\tfrac{1}{2}\zeta_* \tfrac{1}{\sqrt{2\pi}} e^{-(t-2\gamma_*)^2/2} + \zeta_* \tfrac{1}{\sqrt{2\pi}} e^{-(t-\gamma_*)^2/2}\\
    &=  -\tfrac{1}{2} e^{-t^2/2} -\tfrac{1}{2} e^{-(t-2\gamma_*)^2/2} +  e^{-(t-\gamma_*)^2/2}\\
    &= e^{-\frac{1}{2}t^2} \left( -\tfrac{1}{2} -\tfrac{1}{2} e^{-(-4t\gamma_* + 4\gamma_*^2)/2} +  e^{-(-2t\gamma_* + \gamma_*^2)/2} \right)\\
    &= -\tfrac{1}{2} -\tfrac{1}{2} e^{-(-4t\gamma_* + 4\gamma_*^2)/2} +  e^{-(-2t\gamma_* + \gamma_*^2)/2} \\
    &= -\tfrac{1}{2} -\tfrac{1}{2} e^{2t\gamma_* - 2\gamma_*^2} +  e^{t\gamma_* - \tfrac{1}{2}\gamma_*^2}\\
    &= -\tfrac{1}{2} -\tfrac{1}{2} \left(e^{t\gamma_*}\right)^2 e^{- 2\gamma_*^2} +  e^{t\gamma_*}e^{ - \tfrac{1}{2}\gamma_*^2}
\end{align*}
Solving with the quadratic equation gives
\begin{align*}
    e^{t\gamma_*} &= e^{\frac{3}{2}\gamma_*^2} \pm e^{2\gamma_*^2}\sqrt{e^{-\gamma_*^2} - e^{-2\gamma_*^2}}\\
    t &= \tfrac{3}{2}\gamma_* + \tfrac{1}{\gamma_*}\log\left(1 \pm \sqrt{1-e^{-\gamma_*^2}}\right)
\end{align*}
We have found the two extreme values of $\tilde{f}(t)$, which we denote
\begin{align*}
    t_- &= \tfrac{3}{2}\gamma_* + \tfrac{1}{\gamma_*}\log\left(1 + \sqrt{1-e^{-\gamma_*^2}}\right)\\
    t_+ &= \tfrac{3}{2}\gamma_* + \tfrac{1}{\gamma_*}\log\left(1 - \sqrt{1-e^{-\gamma_*^2}}\right)
\end{align*}
It remains to show that they are both suprema; that is, that they attain the same value.

We will show that $\tilde{f}(t_-) = -\tilde{f}(t_+)$. To do this, first define
\begin{align*}
    g(x) &:=  \tfrac{1}{2} e^{-x^2/2} +\tfrac{1}{2} e^{-(x-2\gamma_*)^2/2} - e^{-(x-\gamma_*)^2/2}
\end{align*}
so that
\begin{align*}
    \tilde{f}(t) &=  \zeta_* \frac{1}{\sqrt{2\pi}} \int_{x=-\infty}^t g(x)dx
\end{align*}
The result will follow from three facts: (1) That $g$ is symmetric about $x=\gamma_*$, (2) That $\tilde{f}(\gamma_*) = 0$, and (3) That $\tfrac{1}{2}(t_- + t_+) = \gamma_*$. We will prove each of these facts, and then use them to show that $\tilde{f}(t_-) = -\tilde{f}(t_+)$.
\newline\newline
\textbf{The function $g$ is symmetric.} To show that $g$ is symmetric about $x=\gamma_*$, we will show that $g(\gamma_* + x) = g(\gamma_* - x)$ for all $x$.
\begin{align*}
    g(\gamma_* + x)
    &= -\frac{1}{2}e^{-(\gamma_*+x)^2/2} - \frac{1}{2}e^{-(\gamma_* + x - 2\gamma_*)^2/2} + e^{-(\gamma_* + x - \gamma_*)^2/2}\\
    &= -\frac{1}{2}e^{-(\gamma_*- x)^2/2} - \frac{1}{2}e^{-(\gamma_* - x - 2\gamma_*)^2/2} + e^{-(\gamma_* - x - \gamma_*)^2/2}\\
    &= g(\gamma_* - x)
\end{align*}
\newline\newline
\textbf{The function $\tilde{f}$ is zero at $\gamma_*$.} We compute $\tilde{f}(\gamma_*)$. Our first step is a u-substitution $u = x - \gamma_*$
\begin{align*}
    \tilde{f}(\gamma_*)
    &= \int_{x=-\infty}^{\gamma_*} -\tfrac{1}{2}e^{-x^2/2} - \tfrac{1}{2} e^{-(x-2\gamma_*)^2/2} + e^{-(x-\gamma_*)^2/2}dx\\
    &= \int_{u=-\infty}^{0} -\tfrac{1}{2}e^{-(u+\gamma_*)^2/2}-\tfrac{1}{2}e^{-(u-\gamma_*)^2/2}+e^{-u^2/2}du\\
    &= -\tfrac{1}{2}\left(\Phi(\gamma_*) + \Phi(-\gamma_*)\right) + \Phi(0)
\end{align*}
Recall that $\Phi(0) = \tfrac{1}{2}$, and that $\Phi(-t) = 1 - \Phi(t)$ for any $t$. This gives us
\begin{align*}
    \tilde{f}(\gamma_*)
    &=-\tfrac{1}{2}\left(\Phi(\gamma_*) + 1 - \Phi(\gamma_*)\right) + \tfrac{1}{2}\\
    &=-\tfrac{1}{2} + \tfrac{1}{2}\\
    &= 0
\end{align*}
\newline\newline
\textbf{The average of roots $t_-$ and $t_+$ is $\gamma_*$.} We will show this fact via direct computation,
\begin{align*}
    \frac{1}{2}(t_- + t_+)
    &= \frac{1}{2}\left(3\gamma_* + \tfrac{1}{\gamma_*}\log\left(\left(1 + \sqrt{1-e^{-\gamma_*^2}}\right)\left(1 - \sqrt{1-e^{-\gamma_*^2}}\right)\right) \right)\\
    &= \frac{1}{2}\left(3\gamma_* + \tfrac{1}{\gamma_*}\log\left(1 - \left(1-e^{-\gamma_*^2}\right)\right) \right)\\
    &= \frac{1}{2}\left(3\gamma_* + \tfrac{1}{\gamma_*}\log\left(e^{-\gamma_*^2}\right) \right)\\
    &= \frac{1}{2}\left(3\gamma_* -\gamma_* \right)\\
    &= \gamma_*
\end{align*}
\newline\newline
\textbf{Conclusion: the two critical points $t_-$ and $t_+$ are both suprema.} We will now show that $\tilde{f}(t_-) = -\tilde{f}(t_+)$. We begin by relating both quantities to $\tilde{f}(\gamma_*)$. Recall that $t_- > \gamma_* > t_+$. We have
\begin{align*}
    \tilde{f}(t_+) &= \tilde{f}(\gamma_*) - \zeta_*\frac{1}{\sqrt{2\pi}}\int_{t_+}^{\gamma_*} g(x)dx\\
    &= - \zeta_*\frac{1}{\sqrt{2\pi}}\int_{t_+}^{\gamma_*} g(x)dx\\
    \tilde{f}(t_-) &= \tilde{f}(\gamma_*) + \zeta_*\frac{1}{\sqrt{2\pi}} \int_{\gamma_*}^{t_-} g(x) dx\\
    &= \zeta_*\frac{1}{\sqrt{2\pi}}\int_{\gamma_*}^{t_-} g(x) dx
\end{align*}
Since $t_- - \gamma_* = \gamma_* - t_+$ and $g$ is symmetric about $\gamma_*$, we conclude that
\begin{align*}
    \tilde{f}(t_+) &= - \tilde{f}(t_-)
\end{align*}

The last thing we need to show is that $\tilde{f}(t_+)$ is positive (and, consequently, that $\tilde{f}(t_-)$ is negative). This is a direct consequence of the facts that $\tilde{f}(t)$ only has two critical points, that $\lim_{t\to\infty}\tilde{f}(t) = \lim_{t\to-\infty}\tilde{f}(t) = 0$ (so we know the function crosses zero at most once, at $\gamma_*$), and that
\begin{align*}
    \tilde{f}(0) &= \zeta_*\left( \tfrac{1}{2}\Phi(0) + \tfrac{1}{2}\Phi(2\gamma_*) - \Phi(\gamma_*) \right)\\
    &\geq 0
\end{align*}
by the concavity of $\Phi$ for $x > 0$. Since $0 < \gamma_*$, this tells us that $\tilde{f}$ is positive for all $t < \gamma_*$, including $t_+$. This completes our proof.
\end{proof}

In order to show the optimality of our conjectured $\nu_{OPT}$, we will write down the KKT conditions for the optimization problem  \eqref{eqn:estimation-optimizationProblem} and find points that satisfy them. The following lemmas establish certain properties of the dual points. The statements of Lemmas \ref{lem:factAboutH} and \ref{lem:factAboutHTilde} are motivated by computations in Lemma \ref{lem:optimalNuForm}. Readers may wish to skip these two lemmas on their first reading, and review them after understanding how they are used in the argument of Lemma \ref{lem:optimalNuForm}. 
\begin{lemma}\label{lem:factAboutH}
Let $\gamma_* > 0$. Define
\begin{align*}
    h(x) &:= \frac{\Phi(t_- - x) - \Phi(t_- - 2\gamma_*)}{\Phi(t_+ - x) - \Phi(t_+ - 2\gamma_*)}\\
    t_+ &:= \tfrac{3}{2}\gamma_* + \tfrac{1}{\gamma_*}\log\left( 1 - \sqrt{1 - e^{-\gamma_*}}\right)\\
    t_- &:= \tfrac{3}{2}\gamma_* + \tfrac{1}{\gamma_*}\log\left( 1 + \sqrt{1 - e^{-\gamma_*}}\right)
\end{align*}
Then,
\begin{align*}
    h(x) \begin{cases}
    \leq \frac{c}{1-c} ~&\text{if}~x < 2\gamma_*\\
    \geq \frac{c}{1-c} ~&\text{if}~ x > 2\gamma_*
    \end{cases}
\end{align*}
for $c$ defined by
\begin{align*}
    c &:= \frac{k}{1+k}\\
    k &:= \lim_{x\to 2\gamma_*}h(x) = e^{-\frac{1}{2}(t_- - 2\gamma_*)^2 + \frac{1}{2}(t_+ - 2\gamma_*)^2}
\end{align*}
We note that this implies $c\in(0, 1)$.
\end{lemma}
\begin{proof}
We will prove this result by first breaking the function $h$ into its numerator and denominator,
\begin{align*}
    h(x) &= \frac{f(x)}{g(x)}\\
    f(x) &:= \Phi(t_- - x) - \Phi(t_- - 2\gamma_*)\\
    g(x) &:= \Phi(t_+ - x) - \Phi(t_+ - 2\gamma_*)
\end{align*}
We begin by noting that, for both $f$ and $g$,
\begin{align*}
    f(x), g(x) 
    & \begin{cases}
    \geq 0 ~\text{if}~x < 2\gamma_*\\
    = 0 ~\text{if}~x = 2\gamma_*\\
    \leq 0 ~\text{if}~ x > 2\gamma_*
    \end{cases}
\end{align*}
which follows from the fact that the normal CDF $\Phi$ is strictly increasing.

We will show that $h(x) > k$ when $x > 2\gamma_*$, and that $h(x) < k$ when $x < 2\gamma_*$. This is equivalent to showing that
\begin{align}
    a(x) &:= f(x) - k\cdot g(x) < 0 \qquad \forall x\neq 2\gamma_* \label{eqn:constraintOnA}
\end{align}
(to show the equivalence, recall that $g(x) < 0$ for $x < 2\gamma_*$). In order to show Eqn \eqref{eqn:constraintOnA}, it suffices to show two things: That
\begin{align}
    a(2\gamma_*) = 0 \label{eqn:firstRequirementOnA}
\end{align}
and that
\begin{align}
    a'(x) \begin{cases}
    \leq 0 ~&\text{if}~ x > 2\gamma_*\\
    \geq 0 ~&\text{if}~ x < 2\gamma_*
    \end{cases} \label{eqn:constraintOnADerivative}
\end{align}
Condition \eqref{eqn:firstRequirementOnA} is satisfied because $f(2\gamma_*) = g(2\gamma_*) = 0$. To show that constraint \eqref{eqn:constraintOnADerivative} is satisfied, we will take the derivative of $a(x)$,
\begin{align}
    a'(x) &= \frac{1}{\sqrt{2\pi}} \left( k e^{-\frac{1}{2}(t_+ - x)^2} - e^{-\frac{1}{2}(t_- - x)^2} \right)\label{eqn:aPrime}\\
    &= \frac{1}{\sqrt{2\pi}} \left( e^{-\frac{1}{2}(t_- - 2\gamma_*)^2 + \frac{1}{2}(t_+ - 2\gamma_*)^2-\frac{1}{2}(t_+ - x)^2} - e^{-\frac{1}{2}(t_- - x)^2}\right)\notag\\
    &= \frac{1}{\sqrt{2\pi}} e^{-\frac{1}{2}(t_- - 2\gamma_*)^2}\left( \frac{e^{-\frac{1}{2}(t_+ - x)^2}}{e^{-\frac{1}{2}(t_+ - 2\gamma_*)^2}} -
    \frac{e^{-\frac{1}{2}(t_- - x)^2}}{e^{-\frac{1}{2}(t_- - 2\gamma_*)^2}}\right) \label{eqn:factorizedAPrime}
\end{align}
Our goal is to prove \eqref{eqn:constraintOnADerivative}, which only requires information about the sign of $a'(x)$. Since the leading factor in \eqref{eqn:factorizedAPrime} is positive, we can ignore it, and it suffices to show that
\begin{align*}
     \frac{e^{-\frac{1}{2}(t_+ - x)^2}}{e^{-\frac{1}{2}(t_+ - 2\gamma_*)^2}} -
    \frac{e^{-\frac{1}{2}(t_- - x)^2}}{e^{-\frac{1}{2}(t_- - 2\gamma_*)^2}} \begin{cases}
        \leq 0 ~&\text{if}~ x > 2\gamma_*\\
        \geq 0 ~&\text{if}~ x < 2\gamma_*
        \end{cases}
\end{align*}
Rearranging the terms of the expression, we see that this is equivalent to showing that
\begin{align}
    b(x) := (t_+ - x)^2 + (t_- - 2\gamma_*)^2 - (t_- - x)^2 - (t_- - 2\gamma_*) 
    \begin{cases}
    \geq 0 &~\text{if}~ x > 2\gamma_*\\
    \leq 0 &~\text{if}~ x < 2\gamma_*
    \end{cases}\label{eqn:constraintOnB}
\end{align}
We have $b(2\gamma_*) = 0$, so it suffices to show that its derivative is nonnegative. We take the derivative with respect to $x$,
\begin{align*}
    b'(x) &= -2(t_+ - x) + 2(t_ - x)\\
    &= 2\left( t_- - t_+ \right)
\end{align*}
Since $t_- > t_+$ for $\gamma > 0$, we see that $b'(x) > 0$. We conclude that $b(x)$ satisfies \eqref{eqn:constraintOnB}, which implies \eqref{eqn:constraintOnA}. This shows the desired result, with $c$ such that
\begin{align*}
    \frac{c}{1-c} &= e^{-\frac{1}{2}(t_- - 2\gamma_*)^2 + \frac{1}{2}(t_+ - 2\gamma_*)^2}\\
    &= k
\end{align*}
Solving for $c$ yields $c = \frac{k}{1+k}$, as claimed. Finally, we see that $\frac{c}{1-c} > 0$, so $c\in (0, 1)$, which completes the proof.
\end{proof}

\begin{lemma}\label{lem:factAboutHTilde}
Let $\gamma_* > 0$ and define
\begin{align*}
    \tilde{h}(x) &:= \frac{\Phi(t_- - x) - \Phi(t_-)}{\Phi(t_+ - x) - \Phi(t_+)}\\
    t_+ &:= \tfrac{3}{2}\gamma_* + \tfrac{1}{\gamma_*}\log\left( 1 - \sqrt{1 - e^{-\gamma_*}}\right)\\
    t_- &:= \tfrac{3}{2}\gamma_* + \tfrac{1}{\gamma_*}\log\left( 1 + \sqrt{1 - e^{-\gamma_*}}\right)
\end{align*}
where $\Phi$ is the standard normal CDF. Then
\begin{align*}
    \tilde{h}(x) \leq \frac{c}{1-c} ~\forall x < 0
\end{align*}
for $c$ defined by
\begin{align*}
    c &:= \frac{k}{1+k}\\
    k &:= \lim_{x\to 2\gamma_*}h(x) = e^{-\frac{1}{2}(t_- - 2\gamma_*)^2 + \frac{1}{2}(t_+ - 2\gamma_*)^2}
\end{align*}
\end{lemma}
\begin{proof}
This proof proceeds similarly to the proof of Lemma \ref{lem:factAboutH}. We begin by rewriting the function $\tilde{h}$,
\begin{align*}
    \tilde{h}(x) &= \frac{\tilde{f}(x)}{\tilde{g}(x)}\\
    \tilde{f}(x) &:= \Phi(t_- - x) - \Phi(t_-)\\
    \tilde{g}(x) &:= \Phi(t_+ - x) - \Phi(t_+)
\end{align*}
We begin by noting that
\begin{align*}
    \tilde{g}(x) \geq 0 ~\forall x < 0
\end{align*}
which follows because the CDF is an increasing function. Consequently, showing that $\tilde{h}(x) \leq k$ for $x < 0$ is equivalent to showing that
\begin{align*}
    \tilde{a}(x) := \tilde{f}(x) - k\cdot \tilde{g}(x) \leq 0 \qquad\forall x < 0
\end{align*}
In order to prove this inequality, it suffices to show two things: that
\begin{align*}
    \tilde{a}(0) = 0
\end{align*}
and that
\begin{align*}
    \tilde{a}'(x) \geq 0 \qquad \forall x < 0
\end{align*}
Clearly $\tilde{a}(0) = 0$, since $\tilde{f}(0) = 0$ and $\tilde{g}(0) = 0$. To show that the derivative is positive for negative values of $x$, we begin by taking the derivative,
\begin{align*}
    \tilde{a}'(x) &= \frac{1}{\sqrt{2\pi}}\left( k e^{-(t_+ - x)^2/2} - e^{-(t_- - x)^2/2} \right)
\end{align*}
Now, we realize that the previous equation is the same as Equation \eqref{eqn:aPrime}, the derivative of $a(x)$ in Lemma \ref{lem:factAboutH}. We showed in that proof that $a'(x) \geq 0$ for all $x < 2\gamma_*$, so we conclude that $a'(x) \geq 0$ for all $x < 0$. The proof is complete.
\end{proof}

The next lemma is the crux of the proof of Lemma \ref{lem:boundingLInfForEstBound}. In this lemma, we find the optimal point for the optimization problem \eqref{eqn:estimation-optimizationProblem}.
\begin{lemma}\label{lem:optimalNuForm}
Let $\zeta_*\in(0,1]$, $\gamma_*\in(0,1]$, and
\begin{align*}
    \nu_* &= (1-\zeta_*)\delta_0 + \zeta_*\delta_{\gamma_*}.
\end{align*}
Define
\begin{align*}
    F_{\nu}(t) = \P_{\mu\sim\nu, X\sim\mathcal{N}(\mu,1)}(X \leq t),
\end{align*}
where $\nu$ is a probability distribution over $\mathbb{R}$. Then the optimal point for the optimization problem \eqref{eqn:estimation-optimizationProblem} is given by
\begin{align*}
    \nu_{OPT} := \arg\min_{\nu:\P(\mu > 0)<\frac{1}{2}\zeta_*}||F_\nu - F_{\nu_*}||_\infty = (1-\tfrac{1}{2}\zeta_*)\delta_0 + \tfrac{1}{2}\zeta_*\delta_{2\gamma_*}
\end{align*}
\end{lemma}
\begin{proof}
We will prove this statement for the case where $\nu$ is a vector in the simplex $\triangle^d$; the continuous case can be recovered by taking the limit as $d\to\infty$.

Let $\nu \in \triangle^d$ be a distribution over a discrete set of points $x_i$, with $x\in \mathbb{R}^d$, and $0, \gamma_*, 2\gamma_*\in x$. We will prove the conclusion of the lemma using the tools of convex optimization: We will write down the Lagrange dual of problem \eqref{eqn:estimation-optimizationProblem}, and present a solution that optimizes the dual.

We begin by writing this problem in standard form,
\begin{align*}
    \begin{array}{ll}
         \text{minimize} &  f(\nu)\\
         \text{subject to} &  \sum_i \nu_i = 1\\
         & \nu \succeq 0\\
         & \sum_{i: x_i > 0} \nu_i \leq \frac{1}{2}\zeta_*
    \end{array}
\end{align*}
where $\nu\in\mathbb{R}^d$ and
\[
f(\nu) = \sup_t \left| \frac{1}{\sqrt{2\pi}} \int_{x=-\infty}^t \sum_{i=1}^d \nu_i e^{-(x-x_i)^2/2} - (1-\zeta_*)e^{-x^2/2} - \zeta_* e^{-(x-\gamma_*)^2/2} dx \right|
\]
Since we will be relying on strong duality, we note that our problem satisfies Slater's condition. For example, one interior feasible point for this problem distributes $\tfrac{1}{4}\zeta_*$ mass equally across entries where $x_i > 0$, and the remaining $1 - \tfrac{1}{4}\zeta_*$ mass on the remaining entries. Since $\zeta_* > 0$, this is an interior point of the feasible region.

To apply strong duality, we begin by writing the  Lagrangian:
\begin{align*}
    L(\nu, \lambda_1, \lambda_\zeta, \lambda_0)
    &= f(\nu) + \lambda_1\left(\sum_i \nu_i - 1\right) + \lambda_\zeta \left(\sum_{i : x_i > 0} \nu_i - \tfrac{1}{2}\zeta_*\right) - \lambda_0^T\nu
\end{align*}
Next, we differentiate with respect to $\nu$,
\begin{align}
    \nabla_\nu L(\nu, \lambda_1, \lambda_\zeta, \lambda_0)
    &= \nabla_\nu f(\nu) + \lambda_1 \mathbf{1} + \lambda_\zeta \mathbf{1}_{i:x_i>0} - \lambda_0 \label{eqn:dualDerivative}
\end{align}
We begin by noting that the gradient $\nabla_\nu f(\nu)$ is not always defined, but that the subgradient is.\footnote{We will abuse notation slightly, and use $\nabla$ when referring to subgradients.} In this case, the KKT conditions say that an optimal point $\nu_{OPT}$ must satisfy the following conditions:
\begin{align}
    \nu_{OPT} &\geq 0\label{eqn:primalFeas1}\\
    \sum_{i: x_i > 0} \nu_{OPT}[i] &\leq \frac{1}{2}\zeta_*\label{eqn:primalFeas2}\\
    \sum_i \nu_{OPT}[i] &= 1\label{eqn:primalFeas3}\\
    \lambda_{\zeta,OPT} &\geq 0 \label{eqn:varPos1}\\
    \lambda_{0,OPT} &\geq 0 \label{eqn:varPos2}\\
    \lambda_{\zeta,OPT} \left(\sum_{i : x_i > 0} \nu_{OPT}[i] - \tfrac{1}{2}\zeta_*\right) &= 0 \label{eqn:compSlack1}\\
    \lambda_{0_i,OPT} \nu_{OPT}[i] &= 0 \qquad \forall i \label{eqn:compSlack2}\\
    \nabla_\nu L(\nu, \lambda_{1,OPT}, \lambda_{\zeta,OPT}, \lambda_{0,OPT})\big|_{\nu=\nu_{OPT}} &\ni 0 \label{eqn:zeroGrad}
\end{align}

To show that $\nu_{OPT}$ is indeed an optimal primal point, we will present a feasible set of dual variables that, along with $\nu_{OPT}$, satisfy conditions \eqref{eqn:primalFeas1}-\eqref{eqn:zeroGrad}. By strong duality, this corresponds to an optimal primal point. We use the following $\nu_{OPT}$ and dual variables.
\begin{align*}
    \nu_{OPT}[i] &= 
        \begin{cases}
        1-\frac{1}{2}\zeta_* ~&\text{if}~ x_i = 0\\
        \frac{1}{2}\zeta_* ~&\text{if}~ x_i = 2\gamma_*\\
        0 ~&\text{otherwise}
        \end{cases}\\
    \lambda_{1,OPT} &= (1-c)\Phi(t_-) - c\Phi(t_+)\\
    \lambda_{\zeta,OPT} &= c\left[\Phi(t_+) - \Phi(t_+ - 2\gamma_*) \right] - (1-c) \left[\Phi(t_-) - \Phi(t_- - 2\gamma_*) \right]\\
    \lambda_{0,OPT}[i] &= 
        \begin{cases}
        c\left[ \Phi(t_+ - x_i) - \Phi(t_+) \right] - (1-c) \left[ \Phi(t_- - x_i) - \Phi(t_-)\right] &~\text{if}~x_i < 0\\
        c\left[\Phi(t_+ - x_i) - \Phi(t_+ - 2\gamma_*) \right] - (1-c) \left[\Phi(t_- - x_i) -  \Phi(t_- - 2\gamma_*)\right] &~\text{if}~ x_i > 0\\
        0 &~\text{if}~x_i = 0
        \end{cases}
\end{align*}
where we define constants depending only on $\gamma_*$,
\begin{align*}
    c &:= \frac{k}{1 + k}\\
    k &:= e^{\frac{1}{2}(t_+ - 2\gamma_*)^2 - \frac{1}{2}(t_- - 2\gamma_*)^2}\\
    t_+ &:= \tfrac{3}{2}\gamma_* + \tfrac{1}{\gamma_*}\log\left( 1 - \sqrt{1 - e^{-\gamma_*}}\right)\\
    t_- &:= \tfrac{3}{2}\gamma_* + \tfrac{1}{\gamma_*}\log\left( 1 + \sqrt{1 - e^{-\gamma_*}}\right) 
\end{align*}
It remains to show that these variables satisfy the KKT conditions. We address each category of KKT condition below.
\newline\newline
\noindent\textbf{Primal feasibility}
The primal feasibility conditions \eqref{eqn:primalFeas1}, \eqref{eqn:primalFeas2} and \eqref{eqn:primalFeas3} are all clearly satisfied by our choice of $\nu_{OPT}$. 
\newline\newline
\noindent\textbf{Dual variable nonnegativity}
Conditions \eqref{eqn:varPos1} and \eqref{eqn:varPos2} require that the dual variables corresponding to inequality constraints are nonnegative. Both conditions follow from Lemma \ref{lem:factAboutH}.

For condition \eqref{eqn:varPos1}, nonnegativity of $\lambda_{\zeta,OPT}$, note that $\lambda_{\zeta,OPT} \geq 0$ is equivalent to
\begin{align*}
    \frac{c}{1-c}  &\geq \frac{ \Phi(t_-) - \Phi(t_- - 2\gamma_*) }{\Phi(t_+) - \Phi(t_+ - 2\gamma_*)}
\end{align*}
(note the equivalence holds because the denominator of the right hand side is positive). By Lemma \ref{lem:factAboutH}, with $x=0$, this condition is satisfied.

For condition \eqref{eqn:varPos2}, the constraint is clearly satisfied when $x_i=0$, so it remains to consider the two cases $x_i < 0$ and $x_i > 0$. When $x_i < 0$, the constraint $\lambda_{0,OPT}[i] \geq 0$ is equivalent to
\begin{align*}
    \frac{c}{1-c} &\geq \frac{\Phi(t_- - x_i) - \Phi(t_-)}{\Phi(t_+-x_i) - \Phi(t_+)}
\end{align*}
By Lemma \ref{lem:factAboutHTilde}, this constraint is satisfied for our chosen value of $c$.

When $x_i > 0$, the constraint is clearly satisfied for $x_i = 2\gamma_*$ (since $\lambda_{0,OPT} = 0$ in that case). Otherwise, the constraint is equivalent to
\begin{align*}
    c\left[\Phi(t_+ - x_i) - \Phi(t_+ - 2\gamma_*)\right] \geq (1-c) \left[ \Phi(t_- - x_i) - \Phi(t_- - 2\gamma_*)\right]
\end{align*}
which implies the following system of inequalities for the value of $\frac{c}{1-c}$
\begin{align*}
    \frac{c}{1-c} \geq \frac{\Phi(t_- - x_i) - \Phi(t_- - 2\gamma_*)}{\Phi(t_+ - x_i) - \Phi(t_+ - 2\gamma_*)} \qquad &\text{if}~ x_i < 2\gamma_*
\end{align*}
and
\begin{align*}
    \frac{c}{1-c} \leq \frac{\Phi(t_- - x_i) - \Phi(t_- - 2\gamma_*)}{\Phi(t_+ - x_i) - \Phi(t_+ - 2\gamma_*)}\qquad&\text{if}~ x_i > 2\gamma_*
\end{align*}
By Lemma \ref{lem:factAboutH}, these inequalities are satisfied.
\newline\newline
\noindent\textbf{Complementary slackness}
The complementary slackness conditions \eqref{eqn:compSlack1} and \eqref{eqn:compSlack2} are both satisfied because of the structure of $\nu_{OPT}$ and $\lambda_{0,OPT}$. Condition \eqref{eqn:compSlack1} is satisfied because
\begin{align*}
    \sum_{i:x_i > 0}\nu_{OPT} = \frac{1}{2}\zeta_*
\end{align*}
Condition \eqref{eqn:compSlack2} is satisfied because $\nu_{OPT}[i] = 0$ at all but two values of $i$; at those values (when $x_i = 0$ and $x_i = 2\gamma_*$), we have $\lambda_{0,OPT}[i] = 0$.
\newline\newline
\textbf{Zero is in the subgradient of the Lagrangian}
The KKT conditions require that zero must be in the subgradient of the Lagrangian, evaluated at the optimal set of variables. Recall that the gradient of the Lagrangian is given by
\begin{align*}
    \nabla_\nu L(\nu, \lambda_1, \lambda_\zeta, \lambda_0)
    &= \nabla_\nu f(\nu) + \lambda_1 \mathbf{1} + \lambda_\zeta \mathbf{1}_{i:x_i>0} - \lambda_0
\end{align*}
We therefore start by computing the subgradient of the objective function $f$. 
Recall that, if $f(\nu) = \sup_t f_t(\nu)$ for functions $f_t(\nu)$ indexed by $t$, and if $\mathcal{I}(\nu) = \{ t\in \mathbb{R} ~|~ f_t(\nu) = f(\nu)\}$ is the set of indices for which the sup is attained, then the subgradient of $f$ contains the convex combination of the subgradients of the ``active'' functions whose indices are in $\mathcal{I}(\nu)$,
\begin{align*}
    \textbf{conv}\bigcup_{t\in \mathcal{I}(\nu)} \partial f_t (\nu) \subseteq \partial \sup_{t\in \mathbb{R}}f_t(\nu)
\end{align*}
In our case, we have
\begin{align*}
    f_t(\nu)
    &= \left|\frac{1}{\sqrt{2\pi}} \int_{x=-\infty}^t \sum_{i=1}^d \nu_i e^{-(x-x_i)^2/2} - (1-\zeta_*)e^{-x^2/2} - \zeta_* e^{-(x-\gamma_*)^2/2} dx \right|
\end{align*}
Let $t_*$ be a value of $t$ where the supremum is attained. We have
\begin{align*}
    \nabla_\nu f_{t_*}(\nu)
    &= \nabla_\nu \left| \frac{1}{\sqrt{2\pi}} \int_{x=-\infty}^{t_*} \sum_{i=1}^d \nu_i e^{-(x-x_i)^2/2} - (1-\zeta_*)e^{-x^2/2} - \zeta_* e^{-(x-\gamma_*)^2/2} dx \right|\\
    &=: \nabla_\nu g(\nu, t_*)
\end{align*}
Now, we can compute a subset of the subgradient,
\begin{align*}
    \nabla_\nu f(\nu)
    &\supseteq \textbf{conv} \bigcup_{t_*\in \mathcal{I(\nu)}}\text{sign}(g(\nu, t_*))\nabla_\nu \frac{1}{\sqrt{2\pi}} \int_{x=-\infty}^{t_*} \sum_{i=1}^d \nu_i e^{-(x-x_i)^2/2} - (1-\zeta_*)e^{-x^2/2} - \zeta_* e^{-(x-\gamma_*)^2/2} dx\\
    &= \textbf{conv} \bigcup_{t_*\in \mathcal{I(\nu)}}\text{sign}(g(\nu, t_*))  \Phi(t_* - \mathbf{x})  dx
\end{align*}
Next, we evaluate this derivative at $\nu=\nu_{OPT}$. For our choice of $\nu_{OPT}$, Lemma \ref{lem:tPlusAndMinus} tells us the values of $t_*$,
\begin{align*}
    t_* &= \tfrac{3}{2}\gamma_* + \tfrac{1}{\gamma_*}\log\left( 1 \pm \sqrt{1-e^{-\gamma_*^2}} \right)
\end{align*}
where there are two roots: $t_+$ sets $\text{sign}(g(\nu_{OPT}, t_+))=1$, and $t_-$ sets $\text{sign}(g(\nu_{OPT}, t_-))=-1$. This implies that the subgradient evaluated at $\nu_{OPT}$ contains the convex combination
\begin{align*}
    \nabla_\nu f(\nu)\big|_{\nu=\nu_{OPT}} &\supseteq p\Phi(\mathbf{1}t_+ - \mathbf{x}) + (1-p)\Phi(\mathbf{1}t_- - \mathbf{x})\\
    \nabla_\nu f(\nu)[i]\big|_{\nu=\nu_{OPT}} &\supseteq p\Phi(t_+ - x_i) + (1-p)\Phi(t_- - x_i)
\end{align*}
for some $p\in [0,1]$.

Recall that our goal is to satisfy the KKT condition \eqref{eqn:zeroGrad}, that zero is in the subgradient of the Lagrangian. Having found the subgradient of the objective function, we see this corresponds to showing that
\begin{align*}
    \nabla_\nu L(\nu, \lambda_{1,OPT}, \lambda_{\zeta,OPT}, \lambda_{0,OPT})\big|_{\nu=\nu_{OPT}}[i]
    &\supseteq
    \begin{cases}
    p\Phi(t_+ - x_i) + (1-p)\Phi(t_- - x_i) + \lambda_{1,OPT} + \lambda_{\zeta,OPT} - \lambda_{0,OPT}[i] ~&\text{if}~x_i > 0\\
    p\Phi(t_+ - x_i) + (1-p)\Phi(t_- - x_i) + \lambda_{1,OPT} - \lambda_{0,OPT}[i] ~&\text{if}~x_i \leq 0
    \end{cases}\\
    &= 0
\end{align*}
Choosing $p=c$ (which we know is in $[0, 1]$ from Lemma \ref{lem:factAboutH}), and using the values of $\lambda_{1,OPT}$, $\lambda_{\zeta,OPT}$ and $\lambda_{0,OPT}$ we have chosen, we see that this element of the subgradient is in fact zero. This proves that our solution satisfies the final KKT condition.
\newline\newline
\textbf{Conclusion}
We have proposed a set of primal and dual variables that satisfy the KKT conditions. Since our problem satisfies the conditions of strong duality, we conclude that our choice of primal variables is optimal. 
\end{proof}

Now, we are ready to prove Lemma \ref{lem:boundingLInfForEstBound}.
\begin{proof} \textit{(Proof of Lemma \ref{lem:boundingLInfForEstBound}).}
We begin by nothing that it suffices to prove this lemma for $\sigma=1$, since $\sigma$ is the scale of $\gamma_*$.

Lemma \ref{lem:optimalNuForm} gives us the form of $\nu_{OPT}$, which lets us write
\begin{align*}
    \min_{\nu : \P(\mu > 0) < \frac{1}{2}\zeta_*}||F_\nu - F_{\nu_*}||_\infty
    &= ||F_{\nu_{OPT}} - F_{\nu_*}||_\infty\\
    &= \sup_t |F_{\nu_{OPT}}(t) - F_{\nu_*}(t)|
\end{align*}
The supremum over all $t$ is lower bounded by the value at some $t$. We choose $t = \tfrac{3}{2}\gamma_* - 1$, which is the first-order Taylor series approximation to $t_+$ from Lemma \ref{lem:tPlusAndMinus}. This choice lets us bound the quantity below by
\begin{align}
    \min_{\nu : \P(\mu > 0) < \frac{1}{2}\zeta_*}||F_\nu - F_{\nu_*}||_\infty 
    &\geq \left|F_{\nu'}\left(\frac{3}{2}\gamma_* - 1\right) - F_{\nu_*}\left(\frac{3}{2}\gamma_* - 1\right)\right|\notag\\
    &\geq F_{\nu'}\left(\frac{3}{2}\gamma_* - 1\right) - F_{\nu_*}\left(\frac{3}{2}\gamma_* - 1\right)\notag\\
    &= \left(1-\tfrac{1}{2}\zeta_*\right)\Phi\left(\frac{3}{2}\gamma_* - 1\right) + \tfrac{1}{2}\zeta_*\Phi\left(\frac{3}{2}\gamma_* - 1-2\gamma_*\right)\notag\\
    &\qquad - \left (1-\zeta_*\right)\Phi\left(\frac{3}{2}\gamma_* - 1\right) - \zeta_*\Phi\left(\frac{3}{2}\gamma_* - 1-\gamma_*\right)\notag\\
    &= \tfrac{1}{2}\zeta_*\Phi\left(\frac{3}{2}\gamma_* - 1\right) + \tfrac{1}{2}\zeta_*\Phi\left(-\frac{1}{2}\gamma_* - 1\right) - \zeta_*\Phi\left(\frac{1}{2}\gamma_* - 1\right)\notag\\
    &= \frac{1}{2}\zeta_*\left( \Phi\left(\frac{3}{2}\gamma_* - 1\right) + \Phi\left(-\frac{1}{2}\gamma_* - 1\right) - 2\Phi\left(\frac{1}{2}\gamma_* - 1\right) \right)\label{eqn:lowerBoundAtSmallerRoot}
\end{align}
Next, we apply a Taylor series expansion to the normal CDF $\Phi$. We will take this expansion around $-1$, since we are interested in the behavior for small $\gamma_*$. Let $\phi(x)$ be the normal PDF. We have, via a Taylor series,
\begin{align*}
    \Phi(x) &= \Phi(-1) + \phi(-1)(x+1) + \frac{1}{2}\phi(-1)(x+1)^2 + \frac{1}{6}(x+1)^3 \phi(c)(c^2-1)
\end{align*}
for some $c\in [-1, x]$. We are interested in approximating the CDF at $x=\frac{3}{2}\gamma_* - 1$, so we consider the interval $c\in[-1, -1+\tfrac{3}{2}\gamma_*]$. Since $\gamma_* < 1$, we have
\begin{align*}
    -\frac{1}{\sqrt{2\pi}} \leq \phi(c)(c^2-1) \leq 0
\end{align*}
Taylor's remainder theorem lets us compute upper and lower bounds for the CDF on this interval. The upper bound is given by
\begin{align*}
    \Phi(x) &\geq \min_{c\in[-1, -1 + \frac{3}{2}\gamma_*]}\Phi(-1) + \phi(-1)(x+1) + \frac{1}{2}\phi(-1)(x+1)^2 + \frac{1}{6}(x+1)^3 \phi(c)(c^2-1)\\
    &\geq
    \Phi(-1) + \phi(-1)(x+1) + \frac{1}{2}\phi(-1)(x+1)^2 - \frac{1}{6\sqrt{2\pi}}(x+1)^3 \\
    &=: \Phi_u(x)
\end{align*}
and the lower bound is
\begin{align*}
    \Phi(x) &\leq \max_{c\in[-1, -1 + \frac{3}{2}\gamma_*]}\Phi(-1) + \phi(-1)(x+1) + \frac{1}{2}\phi(-1)(x+1)^2 + \frac{1}{6}(x+1)^3 \phi(c)(c^2-1)\\
    &\leq \Phi(-1) + \phi(-1)(x+1) + \frac{1}{2}\phi(-1)(x+1)^2\\
    &=: \Phi_l(x)
\end{align*}
Applying these bounds to Eqn \eqref{eqn:lowerBoundAtSmallerRoot}, we have
\begin{align*}
    \min_{\nu : \P(\mu > 0) < \frac{1}{2}\zeta_*}||F_\nu - F_{\nu_*}||_\infty
    &\geq \frac{1}{2}\zeta_*\left( \Phi\left(\frac{3}{2}\gamma_* - 1\right) + \Phi\left(-\frac{1}{2}\gamma_* - 1\right) - 2\Phi\left(\frac{1}{2}\gamma_* - 1\right) \right)\\
    &\geq \frac{1}{2}\zeta_*\left( \Phi_l\left(\frac{3}{2}\gamma_* - 1\right) + \Phi_l\left(-\frac{1}{2}\gamma_* - 1\right) - 2\Phi_u\left(\frac{1}{2}\gamma_* - 1\right) \right)\\
    &= \frac{1}{2}\phi(-1)\gamma_*^2 \zeta_*-\frac{13\gamma_*^3\zeta_*}{48\sqrt{2\pi}}
\end{align*}
We can further simplify this,
\begin{align*}
    \min_{\nu : \P(\mu > 0) < \frac{1}{2}\zeta_*}||F_\nu - F_{\nu_*}||_\infty
    &\geq \gamma_*^2 \zeta_* \left(\frac{1}{2}\phi(-1) -\frac{13\gamma_*}{48\sqrt{2\pi}}\right)\\
    &\geq 0.01 \gamma_*^2\zeta_*
\end{align*}
which proves the desired result.
\end{proof}

\subsection{Estimation lower bound (Lemma \ref{lem:finiteSampleEstBound})}\label{sec:estimationLowerBoundProof}
In this section, we prove our finite sample estimation lower bound, Lemma \ref{lem:finiteSampleEstBound}. We begin by stating and proving the main technical lemma we need for this lower bound, a KL divergence calculation for two $\varepsilon$-separated hypotheses. Then, we will prove the lower bound itself, using elements of the standard reduction from estimation to hypothesis testing.

\begin{lemma}\label{lem:KLDivergenceCalculationEstimation}
Let distributions $P_0$ and $P_1$ be mixtures of standard Gaussians defined as
\begin{align*}
    P_0 &~:~ (1-\zeta_*)\mathcal{N}(0, 1) + \zeta_*\mathcal{N}(\gamma_*, 1)\\
    P_1 &~:~(1-\zeta)\mathcal{N}(0, 1) + \zeta\mathcal{N}(\gamma, 1)
\end{align*}
where the parameters for $P_0$ satisfy $\gamma_*\in (0,1)$, $\zeta_*\in (0,\tfrac{1}{2})$, and the parameters for $P_1$ are given by
\begin{align*}
    \zeta &= \zeta_* - \varepsilon\\
    \gamma &= \gamma_* \frac{\zeta_*}{\zeta}
\end{align*}
so that the free parameters are $\gamma_*$, $\zeta_*$ and $\varepsilon$. Let $\varepsilon < \tfrac{2}{3}\zeta_*$. Then the KL divergence between $P_0$ and $P_1$ is bounded above by
\begin{align*}
    KL(P_1, P_0) &\lesssim \varepsilon^2 \gamma_*^4
\end{align*}
\end{lemma}

\begin{proof}
We begin by bounding the KL divergence by the $\chi^2$ divergence,
\begin{align*}
    KL(P_1, P_0) &\leq \chi^2(P_1, P_0)\\
    &= \int \left( \frac{dP_1 - dP_0}{dP_0} \right)^2 dP_0
\end{align*}
We proceed to bound both the numerator and the denominator of the fraction. If $\phi(t)=\frac{1}{\sqrt{2\pi}}e^{-t^2/2}$ is the standard normal PDF, then the denominator is bounded by
\begin{align*}
    dP_0 &= (1-\zeta_*)\phi(t) + \zeta_*\phi(t-\gamma_*)\\
    &\geq (1-\zeta_*)\phi(t)\\
    &\geq \frac{1}{2}\phi(t)
\end{align*}
where we have used the fact that $\zeta_* \in (0, \tfrac{1}{2})$. The numerator is bounded by
\begin{align*}
    dP_1 - dP_0 &= (\zeta - \zeta_*) \phi(t) + \zeta_* \phi(t-\gamma_*)-\zeta\phi(t-\gamma)\\
    &= \phi(t) \left( (\zeta - \zeta_*) + \zeta_* e^{t\gamma_* - \frac{1}{2}\gamma_*^2} - \zeta e^{t\gamma - \frac{1}{2}\gamma^2} \right).
\end{align*}
The factor of $\phi(t)$ is now common to both the numerator and the denominator, so they cancel. The KL divergence is now bounded by
\begin{align}
    KL(P_1, P_0) &\leq \int 4 \left( (\zeta - \zeta_*) + \zeta_* e^{t\gamma_* - \frac{1}{2}\gamma_*^2} - \zeta e^{t\gamma - \frac{1}{2}\gamma^2} \right)^2 dP_0\notag\\
    &=: \int 4\Psi(t)^2 dP_0\notag\\
    &= 4 \E_0[\Psi(X)^2]\notag\\
    &= 4 \left( (1-\zeta_*)\E_{X\sim \mathcal{N}(0, 1)}[\Psi(X)^2] + \zeta_*\E_{X\sim \mathcal{N}(\gamma_*, 1)}[\Psi(X)^2]\right)\label{eqn:klBoundExpectations}
\end{align}
The next step is to bound both of these expectations. To do this, we first expand $\Psi(X)^2$. We have
\begin{align*}
    \Psi(X)^2
    &= (\zeta - \zeta_*)^2 - 2 \zeta(\zeta-\zeta_*)e^{\gamma X - \frac{1}{2}\gamma^2}+\zeta^2 e^{2\gamma X - \gamma^2} \\
    &\qquad + 2 \zeta_* (\zeta-\zeta_*) e^{\gamma_* X - \frac{1}{2}\gamma_*^2} - 2\zeta\zeta_* e^{ \gamma X + \gamma_* X - \frac{1}{2}\gamma^2 -\frac{1}{2}\gamma_*^2 } + \zeta_*^2 e^{2\gamma_* X - \gamma_*^2}
\end{align*}
Note that the random variable $X$ only appears in the form $e^{cX}$. We will evaluate the two expectations in \eqref{eqn:klBoundExpectations} by applying linearity of expectation, and using the moment generating function for a Gaussian random variable. As a reminder, we have
\begin{align*}
    \E_{X\sim\mathcal{N}(0,1)}\left[e^{c X} \right] &= e^{\frac{1}{2}c^2}
\end{align*}
and for the shifted Gaussian,
\begin{align*}
    \E_{X\sim\mathcal{N}(\gamma_*,1)}\left[e^{c X} \right]
    &= \E_{X\sim\mathcal{N}(\gamma_*,1)}\left[e^{c (X-\gamma_*)} e^{c\gamma_*} \right]\\
    &=  e^{c\gamma_*}\E_{X\sim\mathcal{N}(\gamma_*,1)}\left[e^{c (X-\gamma_*)} \right]\\
    &=  e^{c\gamma_*}\E_{X'\sim\mathcal{N}(0,1)}\left[e^{c X'} \right]\\
    &=  e^{c\gamma_*}e^{\frac{1}{2}c^2}\\
    &= e^{c\gamma_* + \frac{1}{2}c^2}
\end{align*}
Now, we are ready to evaluate the expectations in \eqref{eqn:klBoundExpectations},
\begin{align}
    \E_{X\sim\mathcal{N}(0,1)}\left[\Psi(X)^2]
    \right]
    & = (\zeta - \zeta_*)^2 
    - 2 \zeta(\zeta-\zeta_*) 
    + \zeta^2 e^{\gamma^2}
    + 2 \zeta_* (\zeta-\zeta_*)
    - 2\zeta\zeta_* e^{\gamma\gamma_*} 
    + \zeta_*^2 e^{\gamma_*^2}\notag\\
    & = \zeta_*^2\left(e^{\gamma^2}-1\right) - 2\zeta\zeta_*\left( e^{\gamma\gamma_*}-1\right) + \zeta^2\left(e^{\gamma_*^2}-1\right)\label{eqn:centeredExpectation}
\end{align}
and
\begin{align*}
    \E_{X\sim\mathcal{N}(\gamma_*,1)}
    \left[
    \Psi(X)^2
    \right]
    &= 
    (\zeta - \zeta_*)^2 
    - 2 \zeta(\zeta-\zeta_*) e^{  \gamma\gamma_* }
    + \zeta^2 e^{\gamma^2+ 2\gamma\gamma_*}
    + 2 \zeta_* (\zeta-\zeta_*) e^{ \gamma_*^2 }
    - 2 \zeta\zeta_* e^{2\gamma\gamma_*+\gamma_*^2}
    +\zeta_*^2 e^{3\gamma_*^2}
\end{align*}
Our next goal is to upper bound both of these expectations, which will allow us to upper bound the KL divergence in \eqref{eqn:klBoundExpectations}. To bound these expectations, we will use a second order Taylor series approximation to each exponential term, with a third order remainder term. Recall the expansion of $e^x$,
\begin{align*}
    e^x &= 1 + x + \frac{1}{2}x^2 + \frac{1}{6}x^3 e^c
\end{align*}
for some $c$ between $0$ and $x$. Applying this approximation to the exponential terms in \eqref{eqn:centeredExpectation} gives
\begin{align*}
    \E_{X\sim\mathcal{N}(0,1)}\left[\Psi(X)^2]
    \right]
    & = \zeta^2\left(\gamma^2 + \frac{1}{2}\gamma^4 + \frac{1}{6}\gamma^6 e^{c_1}\right) - 2\zeta\zeta_*\left( \gamma\gamma_* + \frac{1}{2}\gamma^2\gamma_*^2 + \frac{1}{6}\gamma^3\gamma_*^3 e^{c_2}\right) + \zeta_*^2\left(\gamma_*^2 + \frac{1}{2}\gamma_*^4 + \frac{1}{6}\gamma_*^6 e^{c_3}\right)
\end{align*}
for $c_1\in[0,\gamma^2]$, $c_2\in[0,\gamma\gamma_*]$ and $c_3\in[0,\gamma_*^2]$. Given our choice of $\gamma$, we have $\zeta\gamma = \zeta_*\gamma_*$. Repeated application of this identity lets us simplify the expression above,
\begin{align*}
    \E_{X\sim\mathcal{N}(0,1)}\left[\Psi(X)^2]
    \right]
    & = \frac{1}{2}\left( \zeta\gamma^2 - \zeta_*\gamma_*^2 \right)^2 + \frac{1}{6}\zeta^2\gamma^6 e^{c_1} - \frac{1}{3}\zeta\zeta_*\gamma^3\gamma_*^3 e^{c_2} + \frac{1}{6}\zeta_*^2\gamma_*^6 e^{c_3}\\
    &\leq \frac{1}{2}\left( \zeta\gamma^2 - \zeta_*\gamma_*^2 \right)^2 + \frac{1}{6}e^{\max(c_1, c_2, c_3)}\left( \zeta\gamma^3 - \zeta_*\gamma_*^3 \right)^2\\
    &= \frac{1}{2}\left( \zeta\gamma^2 - \zeta_*\gamma_*^2 \right)^2 + \frac{1}{6}e^{\gamma^2}\left( \zeta\gamma^3 - \zeta_*\gamma_*^3 \right)^2
\end{align*}
where in the last line we have used the fact that $\gamma > \gamma_*$. Our next goal is to show that the second term, which comes from Taylor's remainder theorem, is negligible compared to the first term. To start, note that we can bound $\gamma^2$ above by an absolute constant,
\begin{align*}
    \gamma^2 
    &= \gamma_*^2 \frac{\zeta_*^2}{\zeta^2}\\
    &\leq \frac{\zeta_*^2}{\zeta^2}\\
    &\leq 9
\end{align*}
because we required $\zeta_* - \zeta < \frac{2}{3}\zeta_*$, which implies that $\tfrac{\zeta_*}{\zeta}< 3$. Next, we argue that the second term is of a smaller order than the first,
\begin{align*}
    (\zeta\gamma^3 - \zeta_*\gamma_*^3)^2
    &= (\zeta_*\gamma^2\gamma_* - \zeta_*\gamma_*^3)^2\\
    &= \zeta_*^2 \gamma_*^2\left( \gamma^2 - \gamma_*^2\right)^2\\
    &= \zeta_*^2 \gamma_*^2\left( (\gamma + \gamma_*)(\gamma - \gamma_*)\right)^2\\
    &= \zeta_*^2 \gamma_*^2\left( (\gamma_*\frac{\zeta_*}{\zeta} + \gamma_*)(\gamma - \gamma_*)\right)^2\\
    &= \zeta_*^2 \gamma_*^2\left(4\gamma_* (\gamma - \gamma_*)\right)^2\\
    &\leq 16\zeta_*^2 \gamma_*^2\left(\gamma_* (\gamma - \gamma_*)\right)^2\\
    &= 16\zeta_*^2\gamma_*^2\left( \gamma - \gamma_* \right)^2\\
    &= 16\left( \zeta_*\gamma\gamma_* - \zeta_*\gamma_*^2 \right)^2\\
    &= 16\left( \zeta\gamma^2 - \zeta_*\gamma_*^2 \right)^2
\end{align*}
We conclude that the expectation can be bounded above in order by just the first term,
\begin{align*}
    \E_{X\sim\mathcal{N}(0,1)}\left[\Psi(X)^2]
    \right]
    & \leq c \left( \zeta\gamma^2 - \zeta_*\gamma_*^2 \right)^2\\
    & = c \varepsilon^2 \gamma_*^4 \left( \frac{\zeta_*}{\zeta_* - \varepsilon} \right)^2\\
    &\leq c' \varepsilon^2 \gamma_*^4
\end{align*}
Next, we argue that the second expectation in \eqref{eqn:klBoundExpectations}, $\E_{X\sim\mathcal{N}(\gamma_*,1)}[\Psi(X)^2]$, is of the same order. Once we show this, then we can conclude that a linear combination of the two terms is also of that order. We begin, as before, with a Taylor series expansion.
\begin{align*}
    \E_{X\sim\mathcal{N}(\gamma_*,1)}
    \left[
    \Psi(X)^2
    \right]
    &= 
    (\zeta - \zeta_*)^2 
    - 2 \zeta(\zeta-\zeta_*) \left(1 + \gamma\gamma_* + \frac{1}{2}\gamma^2\gamma_*^2 + \frac{1}{6}\gamma^3\gamma_*^3 e^{c_1}\right)\\
    &\qquad\qquad
    + \zeta^2 \left( 1 + \gamma^2+ 2\gamma\gamma_* + \frac{1}{2}(\gamma^2+ 2\gamma\gamma_*)^2 + \frac{1}{6}(\gamma^2+ 2\gamma\gamma_*)^3 e^{c_2} \right)\\
    &\qquad\qquad
    + 2 \zeta_* (\zeta-\zeta_*) \left( 1 + \gamma_*^2 + \frac{1}{2}\gamma_*^4 + \frac{1}{6}\gamma_*^6 e^{c_3}\right)\\
    &\qquad\qquad
    - 2 \zeta\zeta_* \left( 1 + 2\gamma\gamma_*+\gamma_*^2 + \frac{1}{2}(2\gamma\gamma_*+\gamma_*^2)^2 + \frac{1}{6}(2\gamma\gamma_*+\gamma_*^2)^3 e^{c_4} \right)\\
    &\qquad\qquad
    +\zeta_*^2 \left( 1 + 3\gamma_*^2 + \frac{9}{2}\gamma_*^4 + \frac{27}{6}\gamma_*^6 e^{c_5}\right)
\end{align*}
The first and second terms from each Taylor series expansion cancel, as a consequence of our choice of $\gamma$. The third order terms combine, again through repeated application of the identity $\zeta\gamma = \zeta_*\gamma_*$, to give the following expression,
\begin{align*}
    \E_{X\sim\mathcal{N}(\gamma_*,1)}
    \left[
    \Psi(X)^2
    \right]
    &= 
    \frac{1}{2}\left(\zeta\gamma^2 - \zeta_*\gamma_*^2\right)^2
    -\frac{1}{3}\zeta(\zeta-\zeta_*)\gamma^3\gamma_*^3 e^{c_1} 
    + \frac{1}{6}\zeta^2(\gamma^2+2\gamma\gamma_*)^3 e^{c_2}\\
    &\qquad\qquad+ \frac{1}{3}\zeta_*(\zeta-\zeta_*)\gamma_*^6 e^{c_3}
    - \frac{1}{3}\zeta\zeta_*(2\gamma\gamma_*+\gamma_*^2)^3 e^{c_4}
    + \frac{27}{6}\zeta_*^2\gamma_*^6 e^{c_5}
\end{align*}
Just like when we bounded the earlier expectation, we note that each of these constants $c_i$ is in fact an absolute constant. Once again, this lets us bound the remainder terms by $c (\zeta\gamma^3 - \zeta_*\gamma_*^3)^2$, which we know is smaller in order than the first term above. We conclude that
\begin{align*}
    \E_{X\sim\mathcal{N}(\gamma_*,1)}
    \left[
    \Psi(X)^2
    \right]
    &\lesssim \left(\zeta\gamma^2 - \zeta_*\gamma_*^2\right)^2\\
    &= \varepsilon^2 \gamma_*^4 \left( \frac{\zeta_*}{\zeta_* - \varepsilon} \right)^2\\
    &\leq 9 \varepsilon^2 \gamma_*^4
\end{align*}
Finally, we substitute these bounds on the expectations into our bound for the KL divergence, \eqref{eqn:klBoundExpectations},
\begin{align*}
    KL(P_1, P_0) 
    &\leq 4 \left( (1-\zeta_*)\E_{X\sim \mathcal{N}(0, 1)}[\Psi(X)^2] + \zeta_*\E_{X\sim \mathcal{N}(\gamma_*, 1)}[\Psi(X)^2]\right)\\
    &\lesssim (1-\zeta_*) \varepsilon^2 \gamma_*^4  + \zeta_* \varepsilon^2 \gamma_*^4 \\
    &\lesssim \varepsilon^2 \gamma_*^4 
\end{align*}
which completes the proof.

\end{proof}

Now, we state and prove the lower bound, Lemma \ref{lem:finiteSampleEstBound}.

\begin{proof} \textit{Proof of Lemma \ref{lem:finiteSampleEstBound}.} 
We begin by noting that it suffices to prove the lemma for $\sigma=1$, since $\sigma$ is the scale of the variable $\gamma_*$. 

We will prove a minimax lower bound on the number of samples taken by any estimator that estimates $\zeta$ within accuracy $\varepsilon$ over the set $A_\varepsilon$, with constant probability. We will use a portion of the standard reduction from estimation to hypothesis testing, as can be found in \cite{Tsybakov}. Specifically, we will prove the statement
\begin{align*}
    \inf_{\widehat{\zeta}_n}\sup_{(\zeta, \gamma)\in A_\varepsilon} \P\left( \left| \widehat{\zeta}_n(X) - \zeta \right| \geq \varepsilon \right) &\gtrsim e^{-n\varepsilon^2 \gamma_*^4 }
\end{align*}
Applying the argument found in Section 2.2 of Tsybakov \yrcite{Tsybakov}, along with Theorem 2.2 of the same, we have the bound
\begin{align*}
     \inf_{\widehat{\zeta}_n}\sup_{(\zeta, \gamma)\in A_\varepsilon} \P\left( \left| \widehat{\zeta}_n(X) - \zeta \right| \geq \varepsilon \right) 
    &\geq \frac{1}{2} e^{-n KL(P_1, P_0)}
\end{align*}
where $P_0$ and $P_1$ are any two parameterizations in $A_\varepsilon$. Choose parameterizations
\begin{align*}
    \zeta_0 &\in (\zeta_* - 2\varepsilon, \zeta_* + 2\varepsilon)\\
    \gamma_0 &\in \left(\tfrac{1}{3}\gamma_*, \tfrac{3}{2}\gamma_*\right)\\ 
    P_0 &= P(\zeta_0, \gamma_0)\\
    P_1 &= P\left(\zeta_0 - \varepsilon, \gamma_0 \frac{\zeta_0}{\zeta_0 - \varepsilon}\right)
\end{align*}
We note that both $P_0$ and $P_1$ are in $A_\varepsilon$, due to the constraint $\varepsilon\in(0, \tfrac{2}{3}\zeta_*)$. Furthermore, by Lemma \ref{lem:KLDivergenceCalculationEstimation}, we have a bound on the KL divergence between $P_0$ and $P_1$. We substitute this into our minimax bound,
\begin{align*}
     \inf_{\widehat{\zeta}_n}\sup_{(\zeta, \gamma)\in A_\varepsilon} \P\left( \left| \widehat{\zeta}_n(X) - \zeta \right| \geq \varepsilon \right) 
    &\gtrsim \varepsilon e^{-n \varepsilon^2\gamma_0^4}
\end{align*}
This bound holds for any choice of $(\zeta_0, \gamma_0)$ in the ranges described above. But note that, for every $(\zeta_0, \gamma_0)$ in this range, the bound is of the same order, specifically
\begin{align*}
     \inf_{\widehat{\zeta}_n}\sup_{(\zeta, \gamma)\in A_\varepsilon} \P\left( \left| \widehat{\zeta}_n(X) - \zeta \right| \geq \varepsilon \right) 
    &\gtrsim \varepsilon e^{-n \varepsilon^2\gamma_*^4}.
\end{align*}
We conclude that any estimator claiming, with constant probability, to estimate $\zeta$ with accuracy better than $\varepsilon$ over $A_\varepsilon$, and in particular on the instance $(\zeta_*, \gamma_*)$, must take at least
\begin{align*}
    n\gtrsim \frac{1}{\varepsilon^2\gamma_*^4}
\end{align*}
samples.
\end{proof}

\section{Experimental Details and Algorithm Implementation}\label{app:implementation}
\subsection{Implementation}
We implemented our estimator in Python. Instead of directly optimizing Eqn \eqref{eqn:estimator}, which we found lacked robustness, we determined the value of $\widehat{\zeta}_n(\gamma)$ via binary search on the unit interval. The algorithm is shown below. At each stage of binary search, the algorithm performs a hypothesis test to decide whether there is $\zeta$ mass above $\gamma$. The hypothesis test is identical to the constraint in Eqn \eqref{eqn:estimator} and the set $S(\zeta, \gamma)\subseteq S(\zeta', \gamma)$ for $\zeta \leq \zeta'$, so this method yields the same results as direct optimization. The optimization was solved using CVXPY \cite{cvxpy, cvxpy_rewriting}, with the ECOS solver and default parameters.

\begin{algorithm}[tb]
\caption{Binary search to return $\widehat{\zeta}_n$}
\label{alg:binarySearch}
\begin{algorithmic}
\STATE {\bfseries Input:} Confidence level $\alpha$, $n$ samples $\{X_i\}_{i=1}^n$, and threshold $\gamma$
\STATE {\bfseries Result:} $\widehat{\zeta}_n$, a lower bound estimate on the number of discoveries above threshold $\gamma$
\STATE Initialize $i_{min} = 0$, $i_{max} = n$, $\tau_{\alpha, n} = \sqrt{\frac{\log(2/\alpha)}{2n}}$
\WHILE{$i_{max} - i_{min} > 1$}
    \STATE $i_{avg} = \lfloor \tfrac{i_{min} + i_{max}}{2}\rfloor$
    \STATE $\zeta = i_{avg}/n$
    \STATE Compute test statistic $T(X; \zeta, \gamma) = \min_{\nu\in S(\zeta, \gamma)} ||\widehat{F}_n - F_{\nu}||_\infty$
    \IF{$T(X; \zeta, \gamma) > \tau_{\alpha, n}$}
        \STATE \emph{// Reject the null hypothesis; conclude there is at least $\zeta$ mass above $\gamma$}
        \STATE $i_{min} = i_{avg}$
    \ELSE
        \STATE $i_{max} = i_{avg}$
    \ENDIF
\ENDWHILE
\STATE $\widehat{\zeta}_n = i_{min} / n$
\STATE {\bfseries return} $\widehat{\zeta}_n$
\end{algorithmic}
\end{algorithm}

\subsection{Code availability and computing infrastructure}
Code implementing our estimator in Python is available at \url{https://github.com/jenniferbrennan/CountingDiscoveries/}. We also provide the data from Hao et. al \yrcite{hao2008drosophila} as a tab-delimited file, to facilitate experiments on their data. Please see the associated README file for an explanation of the data, and an example of loading the data into Python. Experiments were run on an Ubuntu server with 56 cores and 64 GB of RAM.

\subsection{Experimental details for Poisson and Binomial experiments}
 In the Poisson experiment, we took $\mu = 1$ as the null hypothesis and drew $n=100,000$ examples with mean parameters $\lambda_i \sim 0.8\delta_1 + 0.2(\beta(a=2, b=5)*5 + 2)$ (i.e., the alternates means were from a scaled and shifted Beta distribution). In the binomial experiment, we took $n=100,000$ examples with means drawn from $0.9\delta_{0.5} + 0.1(\beta(a=2, b=5)*0.5+0.5)$ and generated test statistics with $t=20$ trials per binomial random variable. Given that $\P(X_i = 20) = 9\cdot 10^{-7}$, while the Bonferroni-adjusted critical value for a test at the 0.05 level is $5\cdot 10^{-7}$, none of the alternate hypotheses in the binomial could be rejected under a FWER guarantee.

\section{Additional Figures Comparing our Estimator to Baselines}\label{app:baselines}
Figure \ref{fig:baselines_no-plugin} compares our estimator to three other estimators for this problem (including only baselines which are guaranteed not to overestimate). The observations are drawn $X_i \sim \mathcal{N}(\mu_i, 1)$, with $\mu_i\sim 0.9\delta_0 + 0.1\delta_{\gamma_*}$. We plot the performance of each estimator as a function of the alternate mean $\gamma_*$, for three values of $n$ (the number of $X_i$ drawn), and two different thresholds. In the settings tested, our estimator gets closest to the true $\zeta_*$ while never overestimating it. Furthermore, our estimator improves as $n$ increases, while the other estimators do not.

\begin{figure}
    \centering
    \includegraphics[width=\textwidth]{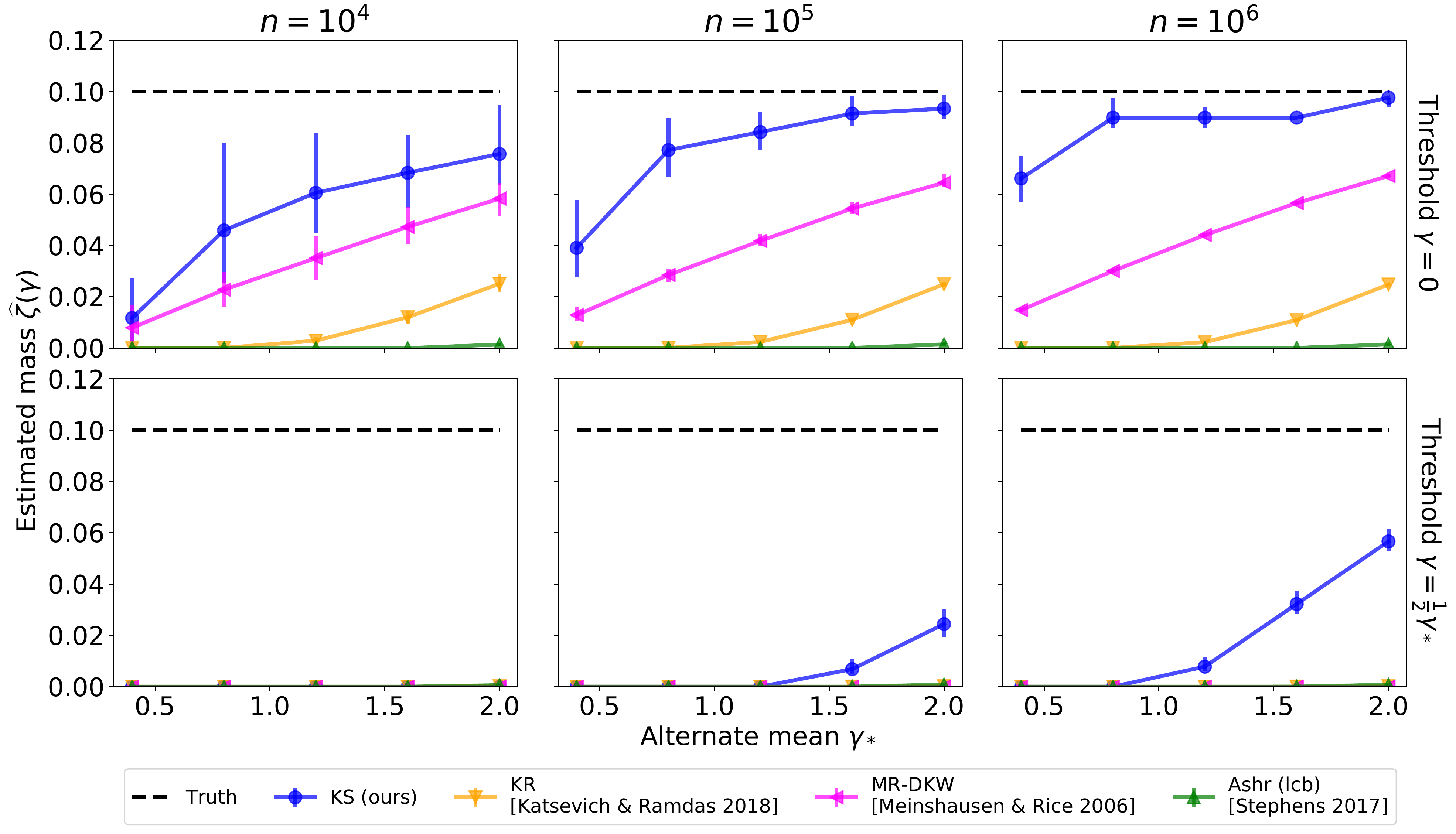}
    \caption{Our estimator compares favorably against the three other methods that satisfy our constraint ($\widehat{\zeta}(\gamma) \leq \zeta_{\nu_*}(\gamma)$ with high probability).}
    \label{fig:baselines_no-plugin}
\end{figure}

\end{document}